\documentclass[final,1p,times,twocolumn]{elsarticle}
\usepackage{amssymb}

\usepackage{dsfont}
\usepackage[english]{babel}
\usepackage[utf8]{inputenc}
\usepackage{algorithm}
\usepackage[noend]{algpseudocode}
\usepackage{ifthen}
\usepackage{caption}
\usepackage{subcaption}
\usepackage{multirow}
\usepackage{xcolor}
\usepackage{booktabs}
\usepackage{amsthm}
\usepackage{amssymb}
\newtheorem{theorem}{Theorem}
\newtheorem{lemma}[theorem]{Lemma}
\usepackage{mathtools}
\DeclarePairedDelimiter\ceil{\lceil}{\rceil}
\DeclarePairedDelimiter\floor{\lfloor}{\rfloor}
\AtBeginDocument{%
  \providecommand\BibTeX{{%
    \normalfont B\kern-0.5em{\scshape i\kern-0.25em b}\kern-0.8em\TeX}}}

\newcommand{\ra}{\rightarrow}

\newcommand{\Exp}[1]{\mathbb{E}\left[#1\right]}
\newcommand{\imp}{\Rightarrow}

\DeclareMathOperator*{\argmax}{arg\,max}
\DeclareMathOperator*{\argmin}{arg\,min}

\newcommand{\ignore}[1]{}
\newcommand{\hide}[1]{}

\newcommand{\remove}[1]{}

\newboolean{showcomments}
\setboolean{showcomments}{true}
\newcommand{\nk}[1]{  \ifthenelse{\boolean{showcomments}}
{ \textcolor{red}{(NK says:  #1)}} {}  }
\newcommand{\jk}[1]{  \ifthenelse{\boolean{showcomments}}
{ \textcolor{red}{(JK says:  #1)}} {}  }
\newcommand{\sj}[1]{  \ifthenelse{\boolean{showcomments}}
{ \textcolor{red}{(SJ says:  #1)}} {}  }
\newcommand{\sg}[1]{  \ifthenelse{\boolean{showcomments}}
{ \textcolor{red}{(SG says:  #1)}} {}  }
\newcommand{\newtext}[1]{#1}

\journal{Performance Evaluation}
\begin{document}

\begin{frontmatter}

%% Title, authors and addresses

%% use the tnoteref command within \title for footnotes;
%% use the tnotetext command for theassociated footnote;
%% use the fnref command within \author or \address for footnotes;
%% use the fntext command for theassociated footnote;
%% use the corref command within \author for corresponding author footnotes;
%% use the cortext command for theassociated footnote;
%% use the ead command for the email address,
%% and the form \ead[url] for the home page:
%% \title{Title\tnoteref{label1}}
%% \tnotetext[label1]{}
%% \author{Name\corref{cor1}\fnref{label2}}
%% \ead{email address}
%% \ead[url]{home page}
%% \fntext[label2]{}
%% \cortext[cor1]{}
%% \affiliation{organization={},
%%             addressline={},
%%             city={},
%%             postcode={},
%%             state={},
%%             country={}}
%% \fntext[label3]{}

\title{Sequential community mode estimation}

%% use optional labels to link authors explicitly to addresses:
%% \author[label1,label2]{}
%% \affiliation[label1]{organization={},
%%             addressline={},
%%             city={},
%%             postcode={},
%%             state={},
%%             country={}}
%%
%% \affiliation[label2]{organization={},
%%             addressline={},
%%             city={},
%%             postcode={},
%%             state={},
%%             country={}}

\author[inst1]{Shubham Anand Jain}

\affiliation[inst1]{organization={Department of Electrical Engineering},%Department and Organization
            addressline={IIT Bombay}, 
            %city={Mumbai},
            %postcode={400076}, 
            %state={Maharashtra},
            country={India}
            }

\author[inst1]{Shreyas Goenka}
\author[inst1]{Divyam Bapna}
\author[inst1]{Nikhil Karamchandani}
\author[inst1]{Jayakrishnan Nair}

% \author[inst1,inst2]{Author Three}

% \affiliation[inst2]{organization={Department Two},%Department and Organization
%             addressline={Address Two}, 
%             city={City Two},
%             postcode={22222}, 
%             state={State Two},
%             country={Country Two}}

\begin{abstract}
%% Text of abstract
We consider a population, partitioned into a set of communities, and study the problem of identifying the largest community within the population via sequential, random sampling of individuals. There are multiple sampling domains, referred to as \emph{boxes}, which also partition the population. Each box may consist of individuals of different communities, and each community may in turn be spread across multiple boxes. 
%each of which includes the individuals belonging to some subset of communities. 
The learning agent can, at any time, sample (with replacement) a random individual from any chosen box; when this is done, the agent learns the community the sampled individual belongs to, and also whether or not this individual has been sampled before. The goal of the agent is to minimize the probability of mis-identifying the largest community in a \emph{fixed budget} setting, by optimizing both the sampling strategy as well as the decision rule. We propose and analyse novel algorithms for this problem, and also establish information theoretic lower bounds on the probability of error under any algorithm. In several cases of interest, the exponential decay rates of the probability of error under our algorithms are shown to be optimal up to constant factors. The proposed algorithms are further validated via simulations on real-world datasets.
% Lorem ipsum dolor sit amet, consectetur adipiscing elit, sed do eiusmod tempor incididunt ut labore et dolore magna aliqua. Ut enim ad minim veniam, quis nostrud exercitation ullamco laboris nisi ut aliquip ex ea commodo consequat. Duis aute irure dolor in reprehenderit in voluptate velit esse cillum dolore eu fugiat nulla pariatur. Excepteur sint occaecat cupidatat non proident, sunt in culpa qui officia deserunt mollit anim id est laborum.
\end{abstract}

% %%Graphical abstract
% \begin{graphicalabstract}
% \includegraphics{grabs}
% \end{graphicalabstract}

% %%Research highlights
% \begin{highlights}
% \item Research highlight 1
% \item Research highlight 2
% \end{highlights}

\begin{keyword}
%% keywords here, in the form: keyword \sep keyword
mode estimation \sep limited precision sampling \sep sequential algorithms \sep fixed budget \sep multi-armed bandits
%% PACS codes here, in the form: \PACS code \sep code
\PACS 0000 \sep 1111
%% MSC codes here, in the form: \MSC code \sep code
%% or \MSC[2008] code \sep code (2000 is the default)
\MSC 0000 \sep 1111
\end{keyword}

\end{frontmatter}

%% \linenumbers

%% main text
\section{Introduction}

Several applications in online learning involve sequential sampling/polling of an underlying population. A classical learning task in this space is \emph{online cardinality estimation}, where the goal is to estimate the size of a set by sequential sampling of elements from the set (see, for example, \cite{finkelstein1998, budianu2006sensors, Bres15}). The key idea here is to use `collisions,' i.e., instances where the same element is sampled more than once, to estimate the size of the set. Another recent application is \emph{community exploration}, where the goal of the learning agent is to sample as many distinct elements as possible, given a family of sampling distributions/domains to poll from (see \cite{chen2018community,Bubeck2013}).

In this paper, we focus on the related problem of \emph{community mode estimation}. Here, the goal of the learning agent is to estimate the largest community within a population of individuals, where each individual belongs to a unique community. The agent has access to a set of sampling domains, referred to as \emph{boxes} in this paper, which also partition the population. The agent can, at any sampling epoch, choose which box to sample from. Having chosen one such box to sample from, a random individual from this box gets revealed to the agent, along with the community that individual belongs to. After a fixed budget of samples is exhausted, the learning agent reveals its estimate of the largest community (a.k.a., the community mode) in the population. The goal of the agent is in turn to minimize the probability of mis-identifying the community mode, by optimizing (i) the policy for sequential sampling of boxes, and (ii) the decision rule that determines the agent's response as a function of all observations.

One application that motivates this formulation is election polling. In this context, communities might correspond to the party/candidate an individual votes for, while boxes might correspond, for instance, to different cities/states that individuals reside in. In this case, community mode identification corresponds to predicting the winning party/candidate. A related (and contemporary) application is the detection of the dominant strain of a virus/pathogen within a population of infected individuals. Here, communities would correspond to different strains, and boxes would correspond to different regions/jurisdictions.

Another application of a different flavour is as follows. Consider a setting where an agent interacts with a database which has  several entries, each with an associated label, and the agent is interested in identifying the most represented label in the database. For concreteness, consider a user who polls a
movie recommendation engine which hosts a large catalogue of movies, each belonging to a particular genre, to discover the most prevalent genre in the catalogue.\footnote{Other relevant objectives, such as discovering the most popular genre in terms of ratings or the genre most `rewarding' for the user, can be incorporated with some modifications to the framework studied here.} In each round, the user might provide a genre (\textit{community}) to the
recommendation engine which then suggests a movie (\textit{individual}) from that genre
(perhaps based on other user ratings). Depending on the recommendations seen thus far, the user selects the next genre to
poll and so on. Now, either due to privacy considerations or simply the lack of knowledge of all the available genres, it might not be feasible for the user to share the exact genre he/she wants to view in each
round and might only provide coarser directions (\textit{box}). For example, while
there might be specific genres available such as dark comedy, romantic
comedy, slapstick comedy etc., the user might only indicate its choice
as `comedy' and then let the recommendation engine suggest some movie belonging to any of the sub-genres in the broad genre. At one extreme, the user might prefer complete privacy and not suggest any genre in each round, in which case the recommendation engine will have to choose a movie over the entire database. This resembles the \textit{mixed community setting} studied in this paper. The opposite end of the spectrum is where the user does not care about privacy and instead specifies a sub-genre in each round from which the recommendation engine can then suggest a movie. This corresponds to the \textit{separated community setting}. We refer to the intermediate scenario where the user provides coarse directives as the \textit{community-disjoint box setting}. 

The formulation we consider here has some parallels with the classical multi-armed bandit (MAB) problem \cite{lattimore2020bandit}; specifically, the fixed budget best arm identification formulation \cite{Aud10}. Indeed, one may interpret communities in our formulation as arms in an MAB problem. However, there are two crucial differences between the two formulations. The first difference lies in the stochastic behavior of the reward/observation sequence. In the classical MAB problem, each pull of an arm yields an i.i.d. reward drawn from an arm specific reward distribution. However, in the community mode detection problem, the sequence of collisions (or equivalently, the evolution of the number of distinct individuals seen) does not admit an i.i.d. description. (Indeed, whether or not a certain sample from a box results in a collision depends in a non-stationary manner on the history of observations from that box.) The second difference between the two formulations lies in the extent of sampling control on part of the agent. In the MAB setting, the agent can pull any arm it chooses at any sampling epoch. However, in our formulation, the agent cannot sample directly from a community of its choice; it must instead choose a box to sample from, limiting its ability to target specific communities to explore.

In terms of the extent of sampling control that the agent has, the opposite end of the spectrum to the MAB setting is when samples are simply generated by an underlying distribution and the agent can only use these observations to estimate some property of the underlying distribution. This classical problem of property estimation from samples generated from an underlying distribution has a long and rich history. There has been a lot of work recently on characterizing the optimal sample  complexity for estimating various properties of probability distributions including entropy \cite{caferov2015optimal, acharya2016estimating}, support size and coverage \cite{hao2019data, wu2018sample}, and `Lipschitz' properties \cite{hao2019unified} amongst others. Closer to the problem studied in this paper, the problem of mode estimation was originally studied in \cite{chernoff1964estimation, parzen62estimation} with the focus on statistical properties of various estimators such as consistency. More recently, the instance-optimal sample complexity of mode estimation for any discrete distribution was derived in \cite{Shah20}. Our formulation differs from this line of work in the non-i.i.d. nature of the observations as well as the partial ability that the agent has to control the sampling process, by being able to query any box at a given instant.

Our contributions are summarized as follows.
\begin{itemize}
    \item We begin by considering a special case of our model where the entire population is contained within a single box; we refer to this as the \emph{mixed community setting} (see Section~\ref{sec:mixedcomm}). In this setting, the sampling process is not controlled, and the learning task involves only the decision rule. We show that a simple decision rule, based on counting the number of distinct individuals encountered from each community, is optimal, via comparison of an upper bound on the probability of error (mis-identification of the community mode) under the proposed algorithm with an information theoretic lower bound. For this setting, we also highlight the impact of being able to identify sampled individuals (i.e., determine whether or not the sampled individual has been seen before) on the achievable performance in community mode estimation.
    \item Next, we consider the case where each community lies in its own box; the so-called \emph{separated community setting} (see Section~\ref{sec:separatedcomm}). Here, we show that the commonly used approach of detecting pairwise collisions (see~\cite{chen2018community}) is sub-optimal. Next, a near-optimal algorithm is proposed that borrows the sampling strategy of the classical \emph{successive rejects} policies for MABs \cite{Aud10}, but differentiates communities based on the number of distinct individuals encountered \newtext{(which is different from the classical MAB setting where arms are differentiated based on their empirical average rewards)}.
    \item Next, we consider a setting that encompasses both the mixed community as well as the separated community settings; we refer to it as the \emph{community-disjoint box setting} (see Section~\ref{sec:boxcomm}). Here, each community is contained within a single box (though a box might contain multiple communities). For this case, we propose novel algorithms that combine elements from the mixed and separated community settings. Finally, we show how the algorithms designed for the community-disjoint box setting can be extended to the fully general case, where communities are arbitrarily spread across boxes.
    \item Finally, we validate the algorithms proposed on both synthetic as well as real-world datasets (see Section~\ref{sec:experimental_results}).
\end{itemize}

\newtext{We conclude this section by making a comparison between our contributions and the literature on the fixed budget MAB problem. Near optimal algorithms for the fixed budget MAB problem (see, for example, \cite{Aud10,Karnin2013}) follow a sampling strategy of \emph{successive rejection} of arms, wherein the sampling budget is split across multiple phases, and at the end of each phase, a certain number of (worst performing) arms are eliminated from further consideration. Some of our algorithms for the community mode estimation problem follow a similar sampling strategy and eliminate boxes in phases; specifically, we often use the same sampling schedule as in the classical successive rejects algorithm proposed in \cite{Aud10}. However, the elimination criterion we use is different: it is based on the number of distinct individuals seen (so far) from each community. Given that this statistic evolves in a non-stationary Markovian fashion over time, this distinction makes our analysis more complex. 

Our information theoretic lower bounds are inspired by the framework developed in \cite{Kaufmann16a} for the fixed budget MAB problem. However, as before, the key distinction in our proofs stems from the difference in stochastic nature of the observation process: while reward observations for each arm in the classical MAB setup are i.i.d., the number of distinct individuals seen from each community evolves as an absorbing Markov chain in the community mode estimation problem.}   
\section{Problem Formulation}
\label{Sec:Problem}
% Consider a population consisting of $N$ individuals. These individuals are spread across $b$ boxes, which is a real life parallel to individuals living in different cities. Each individual belongs to exactly one community; the parallel to this would be the party a person votes for. We denote the size of box $i$ by $N_i$, the total number of communities by $m$, the communities themselves by $C_1, C_2, ..., C_m$ and the number of individuals belonging to community $j$ in box $i$ as $d_{ij}$. Without loss of generality, we order communities in the following way: $\sum_{i}d_{i1} \geq \sum_{i}d_{i2} \geq ... \geq \sum_{i}d_{im}$. We also assume that the largest community is unique; i.e, $\sum_{i}d_{i1} > \sum_{i}d_{i2}$. Our results do not rely on this assumption - this is just for ease of notation. We are given an oracle, and a budget of $t$ samples. The input to this oracle is a box number, and the output from the oracle is a random individual from this box and the community they belong to. The goal is to find an algorithm and sampling scheme which minimizes the probability of error given this budget $t$.\\\\

Consider a population consisting of $N$ individuals. Each individual belongs to exactly one out of~$m$ communities, labelled $1, 2, \cdots, m.$ Additionally, the population is partitioned across $b$ \emph{sampling domains}, also referred to as `boxes' in this paper. The boxes are labelled $1,2,\cdots,b.$ Our learning goal is to identify, via random sequential sampling of the boxes, the largest community (a.k.a., the community mode). 

%, which is a real life parallel to individuals living in different cities. ; the parallel to this would be the party a person votes for. 
We represent the partitioning of the population across communities and boxes via a $b \times m$ matrix~$D.$ The entry in the~$i$th row and~$j$th column of this matrix, denoted by~$d_{ij},$ equals the number of individuals in box~$i$ who are in community~$j$. Throughout the paper, we refer to $D$ as the \emph{instance} associated with the learning task. Let $d_j := \sum_i d_{ij}$ denote the size of community~$j$, and $N_i := \sum_j d_{ij}$ denote the size of box $i$. %\jk{Changed $N_i$ to $d_j.$ Is this ok? Using index $i$ for boxes and $j$ for communities.} \sj{Added $N_i$ definition - it is the size of box $i$}
%{\color{red} We also assume that, via row and column permutations, we renumber the rows and columns such that $d_{11}$ is the largest element in the matrix.}
%JK: Can move this line to Section 5, IMO.

The learning agent a priori knows only the set of boxes and the set of communities.
%We consider an agent who apriori only knows the number of boxes $b$ and has no other information about the underlying instance encoded by the matrix $D$. 
It can access the population by querying an oracle. The input to this oracle is a box number, and the response from the oracle is a (uniformly chosen) random individual from this box and the community that individual belongs to. Individuals are sampled with replacement, i.e., the same individual can be sampled multiple times. Additionally, we assume that the learning agent is able to `identify' the sampled individual, such that it knows whether (and when) the sampled individual had been seen before.\footnote{Note that this does not require the agent to store a unique identifier (like, say, the social security number) associated with each sampled individual. The agent can simply assign its own \emph{pseudo-identity} to an individual the first time the individual is seen. This sampling model has been applied before in a variety of contexts, including cardinality estimation (see \cite{finkelstein1998, budianu2006sensors}) and community exploration (see \cite{chen2018community}).} For each query, the agent can decide which box to sample based on the oracle responses received thus far. At the end of a fixed budget of~$t$ oracle queries, the agent outputs its estimate~$\hat{h}^* \in [m]$ of the community mode $h^*(D) = \argmax_{j \in [m]} d_j$ in the underlying instance $D.$\footnote{We use the notation $[a:b]$ to denote the set $\{a, a+1, \ldots, b\}$ for any $a, b \in \mathbb{Z}$, $b \geq a.$ 
For $b \in \mathbb{N},$ $[b]:= [1:b].$} The agent makes an error if $\hat{h}^* \notin h^*(D)$, and the broad goal of this paper is to design sequential community mode estimation algorithms that minimize the probability of error. 

Formally, for any $k \in [t]$, a sequential algorithm $\mathcal{A}$ has to specify a box $b_k$ to sample for the $k$th query, this choice being a function of only past observations. 
%At the end of~$t$ queries, the algorithm outputs a community index $\hat{j}^* \in [m]$ as its estimate of the community mode. 
The probability of error for an algorithm $\mathcal{A}$ under an instance $D,$ with a budget of~$t$ oracle queries, is given by $P_e(D,\mathcal{A},t) \overset{\Delta}{=} \mathbb{P}(\hat{h}^* \notin h^*(D))$. An algorithm $\mathcal{A}$ is said to be \emph{consistent} if, for any instance~$D,$ $\lim_{t \ra \infty}P_e(D,\mathcal{A},t) = 0.$ We often suppress the dependence on the budget $t$ and also the algorithm $\mathcal{A}$ (when the algorithm under consideration is clear from the context) when expressing the probability of error, denoting it simply as $P_e(D).$

For notational simplicity, we assume throughout that the instance~$D$ is has a unique largest community, with $h^*(D)$ denoting the largest community; our results easily generalize to the case where $D$ has more than one largest community. In the following sections, for various settings of interest, we prove
%whose probability of error converges to zero as $t \rightarrow \infty$ on all instances $D$ is called a \textit{consistent} algorithm. In this work, our primary goal is to design sequential community mode-estimation algorithms which achieve the lowest possible probability of error after $t$ queries. Also, we
instance-specific upper bounds on the probability of error of our proposed algorithms. \newtext{We are also able to prove information theoretic lower bounds on the probability of error under \emph{any} algorithm (within a broad class of \emph{reasonable} algorithms). In some cases, we show that the exponential decay rate of the information theoretic lower bound with respect to the horizon matches (up to a factor that is logarithmic in the number of boxes) the corresponding decay rate for our algorithm-specific upper bounds; this implies the near optimality of our algorithms.} 
%which depend on the `hardness' of the instance $D$. In several scenarios, we are also able to provide matching lower bounds on the performance of any consistent algorithm. 

{\bf Remark:} As is also the case with algorithms for the fixed budget MAB problem, the probability of error under our proposed algorithms typically decays exponentially with respect to the budget~$t,$ i.e., $P_e(D) \leq \mu(D) e^{-\lambda(D) t},$ where $\mu(D),$ and $\lambda(D)$ are instance (and algorithm) dependent positive constants. Our primary goal would be to characterize and optimize the exponential decay rate $\lambda(D)$ above. With the focus thus being on the decay rate, the value of the exponential pre-factor~$\mu(D)$ in our bounds will often be loose; this is also the case in the fixed budget MAB literature. 

\newtext{{\bf Remark:} It is also important to note that in the classical fixed budget MAB problem, the decay rates associated with the upper bounds on the probability of error under the best known algorithms \emph{do not} match exactly the decay rates corresponding to the best known information theoretic lower bounds: the two decay rates differ by a multiplicative factor that is logarithmic in the number of arms \cite{Carpentier2016}. Given this fundamental gap in the state of the art, it is common practice to refer an algorithm as near optimal if the decay rate associated with its upper bound is a logarithmic (in the number of arms) factor away from the decay rate in the best known information theoretic lower bound. Interestingly, we observe a similar multiplicative mismatch between the decay rates in our upper and lower bound for the community mode estimation problem (as noted above).}

%It is also important to note that in the classical fixed budget MAB setting, the best known algorithms achieve a decay rate that is within a logarithmic (in the number of arms) factor of the decay rate associated with the best known information theoretic lower bounds (see \cite{Aud10}).

The remainder of this paper is organized as follows. We begin by considering the \emph{mixed community setting} in Section~\ref{sec:mixedcomm}, where all individuals belong to a single box ($b = 1$); in this special case, the instance matrix $D$ has a single row. Note that in the mixed community setting, the agent has no control on the sampling process. Next, in Section~\ref{sec:separatedcomm}, we study the opposite end of the spectrum with respect to sampling selectivity, where each community constitutes a unique box ($b = m$); this corresponds to $D$ being a diagonal matrix (up to row permutations). We refer to this special case as the \emph{separated community setting.} Next, in Section~\ref{sec:boxcomm}, we consider the intermediate setting, where each community is entirely contained within a single box. 
%Note that here each community still belongs to a single box, 
This corresponds to each column of $D$ having exactly one non-zero entry. The algorithms presented in this section also extend to the most general case, where each community may be spread across multiple boxes. Finally, in Section~\ref{sec:experimental_results}, we present simulation results that compare the proposed algorithms on both synthetic data as well as several real-world datasets. \newtext{We conclude this section with a summary of our main results.}

\subsection*{Summary of main results}

\newtext{
In Tables~\ref{table:summary_mixed}, \ref{table:summary_separated}, and \ref{table:summary_box}, we present a summary of our results, classified by setting. For ease of presentation, only the decay rates associated with our (upper and lower) bounds on probability of error are mentioned here.

\begin{table}[h]
\footnotesize
    \centering
    \scshape
\caption{Summary of the mixed community setting (decay rates)}
        \label{table:summary_mixed}
    \begin{tabular}{l c c c}
    \hline
    \toprule
        Sampling model & Lower bound & Algorithm & Upper bound\\
        \midrule
        \multirow{2}{*}{Identityless} & $\log\left(\frac{N}{N-\left(\sqrt{d_1}-\sqrt{d_2}\right)^2}\right)$ & \multirow{2}{*}{SFM} & $\log\left(\frac{N}{N-\left(\sqrt{d_1}-\sqrt{d_2}\right)^2}\right)$\\
        & (Theorem \ref{theorem:mixed_identityless_lb_1}) & & (Theorem \ref{theorem:mixed_identityless_ub}) \\
         \midrule
         \multirow{2}{*}{Identity} & {$\log\left(\frac{N}{N-\left(d_1-d_2+1\right)}\right)$} & \multirow{2}{*}{DSM} & $\log\left(\frac{N}{N-(d_1-d_2)}\right)$\\
         & (Theorem \ref{theorem:mixed_identity_lb}) & & (Theorem \ref{theorem:mixed_identity_ub}) \\
        \bottomrule
        \end{tabular}
\normalsize
\end{table}

\begin{table}[t]
\footnotesize
    \centering
    \scshape
\caption{Summary of the separated community setting (decay rates)}
        \label{table:summary_separated}
    \begin{tabular}{c c c}
    \hline
    \toprule
        Lower Bound & Algorithm & Upper Bound\\
        \midrule
         $\frac{3}{H_2\left(D\right)}$ &
         DS-SR & $\frac{1}{\overline{log}(b)H(D)}$ \\
(Theorem~\ref{theorem:lower_bound_separated}) & & (Theorem~\ref{thm:distinctsamplesAudibert})\\
        \bottomrule
        \end{tabular}
\normalsize
\end{table}

\begin{table}[t]
\footnotesize
    \centering
    \scshape
\caption{Summary of the community-disjoint box setting (decay rates)}
        \label{table:summary_box}
    \begin{tabular}{l c c c c}
    \hline
    \toprule
        &Lower Bound & Algorithms & Upper Bound\\
        \midrule
         &{$\min\left(  \frac{\Gamma}{H_2^b\left(D\right)} ,\log\left(\frac{N_1}{N_1-(d_{11}-c_1+1)}\right)\right)$} & \multirow{2}{*}{DS-SR, ENDS-SR} & $\min\left( \frac{1}{\overline{log}(b)H^b(D)},\frac{1}{2\overline{log}(b)}\log\left(\frac{N_1}{N_1-d_{11}+c_1}\right)\right)$\\
         &(Theorems \ref{Thm:BoxLB-MixedComm}, \ref{Thm:LBBox2}) & & (Theorem \ref{Thm:UBBox}) \\
        \bottomrule
        \end{tabular}
\normalsize
\end{table}

Table~\ref{table:summary_mixed} summarizes our results for the mixed-community setting, where for simplicity, we have represented the community sizes as $d_1,d_2,\ldots,d_m,$ with $d_1 > d_2 \geq d_3 \geq \cdots \geq d_m.$ In this case, we consider both an \emph{identityless} sampling model, wherein the identity of the sampled individual is not revealed to the learning agent, as well as the identity-based model described in our problem formulation. As we point out in Section~\ref{sec:mixedcomm}, the decay rate corresponding to the identity-based sampling model exceeds that under the identityless model, indicating that identity information helps to improve the performance of mode identification. Note that the decay rates corresponding to our upper and lower bounds match exactly for the identity-based sampling model, and almost exactly for the identity-based model. Since the mixed-community setting consists of a single box, the multiplicative discrepancy described above between the decay rates in the upper and lower bounds does not arise here.

In Table \ref{table:summary_separated}, we summarize our main results for the separated community setting. Since there is a single community per box here, we once again represent the community/box sizes as $d_1,d_2,\ldots,d_b,$ with $d_1 > d_2 \geq d_3 \geq \cdots \geq d_b.$ The decay rate in our lower bound is expressed in terms of the instance-dependent complexity metric $H_2(D) := \sum_{i=2}^{b}\frac{1}{\log\left(d_1\right) - \log\left(d_i\right)}$, and that in our upper bound is expressed in terms of the related complexity metric~$H(D),$ which is within a $\overline{log}(b) = \frac{1}{2} + \sum_{i=2}^b \frac{1}{i}$ factor of $H_2(D)$ (see Lemma~\ref{theorem:separatedupperlowerH}). 
%In addition, the following holds true:
%\begin{align*}
%H^c(D) &= \underset{i\in [2:b]}{max} %\frac{id_1^2d_i}{(d_1-d_i)^2} \geq %\frac{d_1d_b}{d_1-d_b}\underset{i\in[2:b]}{max} %\frac{i}{\log(d_1)-\log(d_i)} =\frac{d_1d_b}{d_1-d_b}H(D).
%\end{align*}
%%$\left(\frac{H(D)}{2} \leq H_2(D) \leq \overline{log}(b)H(D)\right)$

Table \ref{table:summary_box} summarizes our main results for the community-disjoint box setting. Here, $d_{11}$ denotes the size of the largest community, which is contained in Box~1, $c_1$ denotes the size of the second largest community in Box~1, and for $i \geq 2,$ $c_i$ denotes the size of the largest community in Box~$i.$ The remaining constants in the decay rate expressions are defined in Section~\ref{sec:boxcomm}. The decay rates corresponding to the upper and lower bounds are expressed as a minimum of two terms: the first corresponds to the (sub)task of identifying the box containing the largest community, while the second corresponds to the (sub)task of identifying the largest community within that box. As we elaborate in Section~\ref{sec:boxcomm}, for a certain class of (reasonable) instances, the two decay rates can be shown to be within constant factors of one another.

}
    
% \newpage
\section{Mixed Community Setting}
\label{sec:mixedcomm}
We first consider the mixed community setting, where $b = 1,$ i.e., the instance matrix $D$ has a single row. In other words, the population is completely `mixed' and for each query, the agent obtains a uniformly random sample from the entire population. Thus, the sampling process in this case is uncontrolled, and the learning task is to simply identify the largest community based on the~$t$ samples obtained. 

In the mixed community setting, we also consider an \emph{identity-less} sampling model, wherein the agent only learns the community that the sampled individual belongs to, without any other identifying information. Under this sampling model, the agent cannot tell whether or not an individual who has been sampled has been seen before. This model not only forms a benchmark for our subsequent analysis of identity-based sampling, but is also of independent interest, given its privacy-preserving property.

Throughout this section, since there is a single box, we drop the first index in $d_{ij},$ and represent the instance simply as $D = (d_1,d_1,\cdots,d_m).$ Also, without loss of generality, we order the communities as $d_1 > d_2 \geq d_3 \geq \cdots \geq d_m$.

% Formally, we define the setting as follows:\\
% Let the population be of size $N$, with communities $C_1, C_2, ...C_m$ of sizes $d_1, d_2, ...d_m$ respectively. Every individual belongs to a unique community, and has a unique identifier associated with them. Without loss of generality, we assume $d_1 \geq d_2 \geq ... \geq d_m$. In Addition, we assume that $d_1$ is the unique largest community. We want to minimize the probability of error in a fixed number of moves $t$, where, in each move we can sample any one person of the total population. We assume there is an oracle which gives us the pair $(individual, community)$ at each time-step.

\subsection{Identity-less sampling}
We begin by analysing the identity-less sampling model in the mixed community setting. Note that in this case, the response to each oracle query is community $i$, with a probability proportional to the size of the $i$th community. Thus, the agent receives $t$ i.i.d. samples from the discrete distribution $(p_1,p_2,\cdots,p_m),$ where $p_i = d_i/N.$ Hence, the learning task boils down to the identification of the mode of this distribution, using a fixed budget of $t$ i.i.d. samples.\footnote{The same mode identification problem was considered in the \emph{fixed confidence} setting recently in~\cite{Shah20}.}

%\jk{Motivate the privacy preserving motivation of this setting. Also, best to say identity is not seen in this section (rather than say it is not used in our algorithms), since we also have lower bounds.} \sj{Have removed identity in all the algorithms}

\subsubsection{Algorithm}\hfill\\
We consider a natural algorithm in this setting, which we call the Sample Frequency Maximization (SFM) algorithm: return the empirical mode, i.e., the community which has produced the largest number of samples, with ties broken randomly. One would anticipate that this algorithm is optimal, since the vector $(\hat{\mu}_j(t),\ 1 \leq j \leq m),$ where $\hat{\mu}_j(t)$ denotes the number of samples from community~$j$ over~$t$ oracle queries, is a sufficient statistic for the distribution~$D.$
%This is stated formally as Algorithm~\ref{theorem:mixed_identityless_ub}.\jk{Can omit the algorithm description, IMO. If we omit, need to add a line here saying ties can be broken arbitrarily.} 
The probability of error under the SFM algorithm is bounded from above as follows.
%look at the number of samples of each community, and return a community with the highest number of samples. We show that the probability of error of this algorithm matches the lower bound, and hence is asymptotically optimal for the class of Identity-less algorithms.

\ignore{
\begin{algorithm}[t]
\caption{Sample Frequency Maximization algorithm}\label{alg:samplefrequency}
\begin{algorithmic}[1]
\Procedure{SFM}{$t$}
% \Comment{Estimates largest community with access to $t$ samples}
\State Set $comm\_samples[c] = 0$ for all $c \in communities$
\For{$k = 1, 2, ..t$}
\State get sample $c$ where $c = community$
\State $comm\_samples[c] = comm\_samples[c] + 1$
\EndFor
\State Let $w = \{c : comm\_samples[c] = max_{c' \in communities}comm\_samples[c'] \}$
\State $\hat{h}^*$ = random sample from $w$
\State \textbf{return} $\hat{h}^*$\Comment{The guess for largest community is $\hat{h}^*$}
\EndProcedure
\end{algorithmic}
\end{algorithm}
} % end ignore 

\begin{theorem} 
\label{theorem:mixed_identityless_ub}
Consider the mixed community setting, under the identity-less sampling model. For any instance $D,$ the Sample Frequency Maximization algorithm 
%given in Algorithm~\ref{alg:samplefrequency}, 
has a probability of error upper bounded as
\begin{align*}
    P_e{(D)} \leq (m-1) \left(1-\frac{(\sqrt{d_1}-\sqrt{d_2})^2}{N}\right)^t. %\leq  (m-1)\exp\left(-\frac{t(d_1-d_2)^2}{4Nd_1}\right).
\end{align*}
\end{theorem}
The proof, which follows from a straightforward application of the Chernoff bound, can be found in~\ref{sec:proof_sample_freq_estimate}. Note that the probability of error under the SFM algorithm decays exponentially with the budget~$t,$ the decay rate being (at least) $\log\left(\frac{N}{N-(\sqrt{d_1}-\sqrt{d_2})^2}\right).$ The optimality of this decay rate is established next, via an information-theoretic lower bound on the probability of error under any consistent algorithm. 

\subsubsection{Lower Bound}\hfill\\
%Let us define a \textit{consistent algorithm} as an algorithm that gives the correct answer with probability $1$ on every instance $(d_1, d_2, ... d_m)$ as $t \rightarrow \infty$. Then, we have the following result:
The following theorem establishes an asymptotic lower bound on the probability of error under any consistent algorithm which uses identity-less sampling. Recall that under a consistent algorithm, for any underlying instance $D$ the probability of error converges to zero as $t \ra \infty.$
\begin{theorem}
\label{theorem:mixed_identityless_lb_1}
In the mixed community setting, under the identity-less sampling model, any consistent algorithm on an instance $D$ satisfies
\begin{align*}
    \liminf_{t\rightarrow\infty} \frac{1}{t}\log(P_e(D)) \geq -\log\left(\frac{N}{N-(\sqrt{d_1}-\sqrt{d_2})^2}\right).
\end{align*}
%\sj{Need to replace the words "lower bound" with something else - not sure what yet.}
\end{theorem}
The proof of this theorem, which uses ideas from the proof of \cite[Theorem 12]{Kaufmann16a}, can be found in~\ref{sec:proof_identityless_consistent_lb}. Since the exponential decay rate in the above lower bound matches that in the upper bound corresponding to the SFM algorithm for any instance $D$, it follows that SFM is asymptotically decay-rate optimal (under identity-less sampling).

% \nk{We should probably define $P_e(D)$ formally in the problem setup}
\ignore{
We also present another result which gives non-asymptotic lower bounds on all algorithms (i.e, not assuming consistency). 
\begin{theorem}
\label{theorem:mixed_identityless_lb_2}
For any permutation-invariant algorithm on an instance $(d_1, d_2, ... d_m)$, the probability of error is lower bounded as
\begin{align*}
    P(error) \geq \frac{1}{4}\exp\left(-\frac{t(d_1-d_2)}{N}\log\left(\frac{d_1}{d_2}\right)\right)
\end{align*}
\end{theorem}
\textbf{Proof.} This proof is similar to the proof of Lemma 15 in \cite{Kaufmann16a}. A complete proof is given in Appendix \ref{sec:proof_identityless_permutation_lb}. $\blacksquare$
} %end ignore

\subsection{Identity Sampling}
Having considered the case of identity-less sampling in the previous section, we now revert to the identity-based sampling model described in Section~\ref{Sec:Problem}. We show that identity information can be used to improve the accuracy of  community mode estimation. We begin by proposing and analysing a simple algorithm for community mode estimation, and then establish information-theoretic lower bounds.

%In identity sampling algorithms, we note the community a person belongs to, as well as a unique identifier number for each sample. We use both the components of this information.

\ignore{
\subsubsection{Algorithm}
\begin{algorithm}[t]
\caption{Distinct Samples Maximization algorithm}\label{alg:distinctsamplefrequency}
\begin{algorithmic}[1]
\Procedure{DSM}{$t$}
% \Comment{Estimates largest community with access to $t$ samples}
\State Set $ind\_samples[i] = 0$ for all $i \in individuals$
\State Set $comm\_samples[c] = 0$ for all $c \in communities$
\For{$k = 1, 2, ..t$}
\State get sample $(i, c)$ where $i = individual, c = community$
\If{$ind\_samples[i] = 0$}\Comment{Has not been sampled before}
\State $ind\_samples[i] = ind\_samples[i] + 1$
\State $comm\_samples[c] = comm\_samples[c] + 1$
\EndIf
\EndFor
\State Let $w = \{c : comm\_samples[c] = max_{c' \in communities}comm\_samples[c'] \}$
\State $\hat{h}^*$ = random sample from $w$
\State \textbf{return} $\hat{h}^*$\Comment{The guess for largest community is $\hat{h}^*$}
\EndProcedure
\end{algorithmic}
\end{algorithm}
} %end ignore

\subsubsection{Algorithm}

Under identity-based sampling, we propose a simple \emph{Distinct Samples Maximization} (DSM) algorithm: The DSM algorithm tracks the number of \emph{distinct} individuals seen from each community, and returns the community that has produced the greatest number over the $t$ queries, with ties broken randomly. As before, this is the natural algorithm to consider under identity-based sampling, given that the vector $(S_j(t),\ 1\leq j \leq m)$, where $S_j(t)$ denotes the number of distinct individuals from community~$j$ seen over~$t$ oracle queries, is a sufficient statistic for $D$ (see \cite{budianu2006sensors}). 
% This is formalized as Algorithm~\ref{alg:distinctsamplefrequency}. 
The probability of error under the DSM algorithm is bounded as follows.
\begin{theorem}
\label{theorem:mixed_identity_ub}
In the mixed community setting, for any instance $D,$ the Distinct Samples Maximization (DSM) algorithm 
%given in Algorithm~\ref{alg:distinctsamplefrequency} 
has a probability of error upper bounded as
\begin{align}
    &P_e(D) \leq 2(m-1)\exp\left(-\frac{t\left(d_1-\frac{\sum_{i=2}^{m}d_i}{m-1}\right)^2}{32Nd_1}\right)\quad \text{ for } t\leq min\left\{\frac{d_1+d_m}{2d_1}N, \frac{16Nd_1}{(d_1-d_m)^2}\right\}, \label{eq:community_McDiarmid} \\
    &P_e(D) \leq {\binom{d_1}{d_2}} \left(1-\frac{d_1-d_2}{N}\right)^t {= {\binom{d_1}{d_2}}\exp\left(-t\log\left(\frac{N}{N-d_1+d_2}\right)\right)} 
    %\exp{\left(-\frac{t(d_1-d_2)}{N}\right)} 
    \quad \forall t. \label{eq:community_coupon_collector}
\end{align}
\end{theorem}
Theorem~\ref{theorem:mixed_identity_ub} provides two upper bounds on the probability of error. The bound~\eqref{eq:community_coupon_collector} holds for all values of budget~$t,$ while the bound~\eqref{eq:community_McDiarmid} which is only applicable for small to moderate budget values, tends to be tighter for small values of~$t.$ Note that~\eqref{eq:community_coupon_collector} implies that the probability of error under the DSM algorithm decays exponentially with~$t,$ with decay rate (at least) \ignore{$\log\left(\frac{N-(d_1-d_2)}{N}\right).$}$\log\left(\frac{N}{N-(d_1-d_2)}\right).$ 
%\sj{Fixed typo here - this should be correct decay rate} 
Note that this decay rate exceeds the optimal decay rate under identity-less sampling from Theorem~\ref{theorem:mixed_identityless_lb_1}, since $$d_1 - d_2 > (\sqrt{d_1}-\sqrt{d_2})^2 \imp \log\left(\frac{N}{N-(d_1-d_2)}\right) > \log\left(\frac{N}{N-(\sqrt{d_1}-\sqrt{d_2})^2}\right).$$ This shows that identity information indeed improves the accuracy of  community mode  estimation.

%\textbf{Proof.} 
\begin{proof}
The proof of \eqref{eq:community_McDiarmid} relies on an argument using McDiarmid's inequality, and is given in~\ref{sec:proof_mixed_identity_ub}. The proof of \eqref{eq:community_coupon_collector} is given by a coupon collector style argument. The error probability is upper bounded by the probability of the event that there exists a subset of $d_1 - d_2$ individuals in the largest community $C_1$, such that none of them are sampled in the $t$ queries. Thus we have
%: Choose a subset of size $d_2$ from $C_1$. The error probability is upper bounded by the probability of not seeing any of the remaining $d_1-d_2$ individuals.
\begin{align*}
    P_e(D) 
    %\leq P(S_{t}(\phi_1) \leq d_2) 
    &\leq{\binom{d_1}{d_2}} \left(1-\frac{d_1-d_2}{N}\right)^t .%\\
    %&\leq {\binom{d_1}{d_2}} %\exp\left(-\frac{t(d_1-d_2)}{N}\right)\
\end{align*}
The details can be found in~\ref{sec:proof_mixed_identity_ub}.
%\sj{Here, $S_t(\phi_1)$ wasn't defined. Should we care about really adding this here, or leave that middle expression out? I'm not sure if it really aids understanding enough to explain it}
%
%\nk{We can remove the intermediate expression involving $S_t(\phi_1)$ here. Instead, we should include a more detailed proof in the Appendix. I think that is needed in any case, also to argue the assertion that when there is an error, there will be a subset of size $d_2$ which is not sampled}
%
%Note that for small $t$ the coupon collector result might be weaker than the first due to the large constant, but it is tighter for large $t$ due to the faster decay rate. 
\end{proof}

\subsubsection{Lower Bounds}
Next, we show that the exponential decay rate of the probability of error under the DSM algorithm is (nearly) optimal via an information-theoretic lower bound. %For the lower bound, we restrict attention to algorithms whose output depends only on the number of distinct individuals seen from each community. The motivation for this restriction is that the vector of distinct individuals seen from each community is a sufficient statistic for the vector of true community sizes (see \cite{??}). 
\begin{theorem}
\label{theorem:mixed_identity_lb}
In the mixed community setting, for any consistent algorithm,
%whose output depends only on the number of distinct people seen from each community, 
the probability of error corresponding to an instance~$D$ is bounded below asymptotically as
\begin{align*}
%    P(error) \geq \exp\left(-t\log\left(\frac{N}{N-d_1+d_2-1}\right)\right) \geq \exp\left(-\frac{t(d_1-d_2+1)}{N-d_1+d_2-1}\right)
\liminf_{t \ra \infty} \frac{\log(P_e(D))}{t} \geq - \log\left(\frac{N}{N-(d_1-d_2+1)}\right).
\end{align*}
\end{theorem}
Note that Theorem~\ref{theorem:mixed_identity_lb} implies that the DSM algorithm is nearly decay-rate optimal; the small discrepancy between the decay rate under DSM and that in the lower bound ($(d_1 - d_2)$ replaced by $(d_1-d_2+1)$) stems from the discreteness of the space of alternative instances in our change of measure argument. 
The proof of this Theorem can be found in~\ref{sec:proof_mixed_identity_lb}.
\section{Separated Community Setting}
\label{sec:separatedcomm}

In this section, we consider the \emph{separated community} setting, where each box contains a single and unique community (so that $b = m$). 
Compared to the mixed community setting considered in Section~\ref{sec:mixedcomm}, this setting represents the opposite end of the spectrum with respect to sampling selectivity on part of the agent---the agent can now choose exactly which community to sample from at any time. Note that identity-less sampling is not meaningful in the separated community setting, since the agent can only gauge the size of a community by observing `collisions,' which occur when the same individual is sampled again.
%In the Separated Community Setting, each box has exactly one community, i.e $b = m$. 
%Thus, we can choose the community we want to sample from. We want to identify the largest community among these. Clearly, there is no "identity-less" analogue for this setting; thus, we only consider the identity setting.

At a high level, the separated community setting has connections with the (fixed budget) multi-armed bandit (MAB) problem, with boxes/communities corresponding to arms. However, the reward structure in the separated community setting is different from that in a classical MAB problem; indeed, whether or not a sample taken from any community represents a collision depends on past samples from that community. Nevertheless, we show that tools from the MAB literature can still be adapted to design near-optimal algorithms for estimating the largest community in our setting. %\sj{We started with ``However,'' twice - can this be re-written? Also, this makes it sound like collisions can be made optimal. Maybe it should be - ideas from the MAB literature; here, this is the SR algorithm}

Throughout this section, we denote the size of the community in the $b$th box by $d_b$, dropping the redundant second index since there is only one community in each box. Thus, an instance can be defined by the vector $D = (d_1,d_2,\cdots,d_b).$ WLOG, we order the communities such that $d_1 > d_2 \geq d_3 ... \geq d_b$.

We begin by considering a simple approach, where at each decision epoch, the agent queries a pair of samples from any chosen community, and checks whether or not a collision has occurred, i.e., the same individual has been sampled both times. Since the event of such a (pairwise, consecutive) collision is independent of past samples, and its probability is inversely proportional to the size of the community, this provides a direct mapping to the MAB setting, allowing off-the-shelf MAB algorithms to be applied.\footnote{Note that this approach only looks for `immediate' collisions and does not track collisions across the entire observation history.} However, we find that this approach, which has been used before in the literature (for example, see \cite{chen2018community} for an application of this approach to community exploration), is sub-optimal. Next, we propose and analyse an algorithm that tracks the number of distinct individuals seen from each community, and performs a successive elimination of communities until one `winner' remains. We show that this approach is near-optimal, by comparing its performance to an information-theoretic lower bound. 

\subsection{Algorithms}

We begin by describing the successive rejects (SR) algorithm for fixed-budget MABs, proposed in~\cite{Aud10} for best arm identification. The SR algorithm is known to be near-optimal in this setting. Our algorithms for the estimation of the largest community, which borrow the sampling framework of the SR algorithm, are described next.

\noindent {\bf Successive rejects algorithm:} Consider an MAB problem with $b$ arms. The class of successive rejects (SR) algorithms is parameterized by natural numbers~$K_1,K_2, \cdots,K_{b-1},$ satisfying $0 =: K_0 \leq K_1 \leq K_2 \leq \cdots \leq K_{b-1},$ and $\sum_{j=1}^{b-2} K_j + 2 K_{b-1} \leq t,$ where $t$ denotes the budget/horizon. The algorithm proceeds in $b-1$ phases, with one arm being rejected from further consideration at the end of each phase. Specifically, in Phase~$r,$ the $b-r+1$ surviving arms are each pulled $K_r-K_{r-1}$ times. At the end of this round, the worst performing\footnote{In the classical setting where the best arm is defined as the one with the greatest mean reward, the worst performing arm would be the one with the smallest empirical mean estimate.} surviving arm, based on the $K_r$ samples seen so far, is rejected. The output of the algorithm is the arm that survives rejection at the end of Phase~$b-1.$ The original SR algorithm proposed in \cite{Aud10} used $K_r \propto \frac{t-b}{b-r+1},$ so that \begin{equation}
\label{eq:SR_roundlengths}
    K_r = \left\lceil\frac{1}{\overline{log}(b)}\frac{t-b}{b-r+1}\right\rceil,
\end{equation} where $\overline{log}(b) = \frac{1}{2}+\sum_{i=2}^{b}\frac{1}{i}$. Other SR variants, including \emph{uniform exploration} ($K_r = \floor{t/b}$ for $1 \leq r \leq b-1$) and \emph{successive halving} (see \cite{Karnin13}) have also been considered in the literature. In the remainder of this paper, when we refer to the SR algorithm, we mean the specific algorithm proposed in \cite{Aud10}, with phases defined via \eqref{eq:SR_roundlengths}. 
%Consider a consecutive pair of samples, $s_1$ and $s_2$. If $s_1 = s_2$, we call it a \textit{collision}. We now present an algorithm that is based on the number of collisions we observe. We note that this translates directly to a multi-armed bandit problem, where getting a collision is a bernoulli random variable with probability $\frac{1}{N_i}$, if we are sampling the $i^{th}$ box. In \cite{Aud10}, an optimal multi-round allocation scheme for best arm identification in multi-armed bandit was proposed. In this scheme, $K_r = \left\lceil\frac{1}{\overline{log}(m)}\frac{t-m}{m-r+1}\right\rceil$ samples are used upto the $r^{th}$ round by the communities surviving upto that round, where $\overline{log}(m) = \frac{1}{2}+\sum_{i=2}^{m}\frac{1}{i}$. Thus, in each round, $(m-r+1)*(K_r - K_{r-1})$ samples are used in total, and each community in specific gets sampled $K_r-K_{r-1}$ times per round. We call this the Audibert sampling scheme.
%
\begin{algorithm}[t]
\caption{Consecutive-collision SR algorithm}\label{alg:collisionsaudibert}
\begin{algorithmic}[1]
\State Set $\mathcal{B} = [b]$ \Comment{Set of surviving boxes}
\State Set $K_0 = 0$, $K_r=\ceil{\frac{1}{\overline{\log}(b)}\frac{t/2-b}{b-r+1}} \quad (1 \leq r \leq b-1$)
\For{$r = 1, 2, ..b-1$}
\State For each box in $\mathcal{B},$ perform $(K_r - K_{r-1})$ sample pairs
\State Set $C_{i}^r$ as number of consecutive (within disjoint sample pairs) collisions in box~$i \in \mathcal{B}$
\State $\mathcal{B} = \mathcal{B} \setminus \{\argmax_{i \in \mathcal{B}} C_{i}^r\}$ \hspace{2 cm} (ties broken randomly) %\Comment{Eliminate lowest ranked box}
\EndFor
\State \textbf{Return} $\hat{h}^*$ = lone surviving box in $\mathcal{B}$  
\end{algorithmic}
\end{algorithm}

\noindent {\bf Consecutive-collision SR algorithm:} In this algorithm, we map the largest community identification problem to an MAB best arm identification problem. Each community is treated as an arm, and an arm pull consists of two samples drawn from that community. The reward is binary, being~1 if the arm pull does not result in a collision, and~0 if it does. Thus, the mean reward associated with arm (community)~$i$ equals $1-\frac{1}{d_i},$ so that the best arm (the one with the highest mean reward) corresponds to the largest community. Note that since each arm pull corresponds to~2 samples, the budget of the MAB reformulation equals~$t/2.$ On this MAB reformulation, we apply the SR algorithm of \cite{Aud10} to identify the largest community; this is formalized as Algorithm~\ref{alg:collisionsaudibert}. Adapting the proof of \cite[Theorem 2]{Aud10} for our setting yields the following upper bound on the probability of error under the Consecutive-collision SR (CC-SR) algorithm.
\begin{theorem}
\label{theorem:collision_audibert}
In the separated community setting, for any instance~$D,$ the Consecutive-collision SR (CC-SR) algorithm given in Algorithm~\ref{alg:collisionsaudibert} has a probability of error that is upper bounded as
\begin{align*}
    P_e(D) \leq \frac{b(b-1)}{2}\mathrm{exp}\left(-\frac{(t/2-b)}{4\overline{log}(b){H}^c(D)}\right),
\end{align*}
where $\Delta_{i} = \frac{1}{d_i}-\frac{1}{d_1}$, and ${H}^c(D) = \underset{i\in [2:b]}{max}\frac{i\Delta_{i}^{-2}}{d_i}$.
\end{theorem}
The proof of Theorem~\ref{theorem:collision_audibert}, which uses the Chernoff bound to concentrate the number of consecutive collisions from each community, can be found in~\ref{sec:proof_collision_audibert}.
%\sj{Discrepancy between above and Aud10 to be investigated.}

\noindent {\bf Distinct Samples SR algorithm:} We now present an algorithm that ranks communities by the number of distinct individuals seen. Note that this involves tracking collisions across the entire observation history of each community. Specifically, we use the same sampling strategy as the SR algorithm, and at the end of each phase, eliminate from further consideration that community which has produced the least number of distinct individuals so far.\footnote{Note however that in the original SR algorithm for MABs, the cumulative reward from each arm has i.i.d. increments. In the present setting however, the cumulative number of distinct individuals seen from any community does not have i.i.d. increments.} This algorithm, which we refer to as the Distinct Samples SR (DS-SR) algorithm, is stated formally as Algorithm~\ref{alg:distinctsamplesaudibert}.

\begin{algorithm}[t]
\caption{Distinct Samples SR algorithm (separated community setting)}\label{alg:distinctsamplesaudibert}
\begin{algorithmic}[1]
\State Set $\mathcal{B} = [b]$ \Comment{Set of surviving boxes}
\State Set $K_0 = 0$, $K_r=\ceil{\frac{1}{\overline{\log}(b)}\frac{t-b}{b-r+1}} \quad (1 \leq r \leq b-1$)
\For{$r = 1, 2, ..b-1$}
\State Sample each box in $\mathcal{B},$ $K_r - K_{r-1}$ times
\State Set $S_{i}^r$ as number of distinct individuals seen so far from box~$i \in \mathcal{B}$
\State $\mathcal{B} = \mathcal{B} \setminus \{\argmin_{i \in \mathcal{B}} S_{i}^r\}$ \hspace{2 cm} (ties broken randomly) %\Comment{Eliminate lowest ranked box}
\EndFor
\State Set $\hat{b}$ as lone surviving box in $\mathcal{B}$
\State \textbf{Return} $\hat{h}^* =$ lone surviving box in $\mathcal{B}$ 
\end{algorithmic}
\end{algorithm}

\begin{theorem}
\label{thm:distinctsamplesAudibert}
In the separated community setting, for any instance~$D$ the Distinct Samples SR (DS-SR) algorithm given in Algorithm~\ref{alg:distinctsamplesaudibert} has a probability of error that is upper bounded as
\begin{align*}
    P_e(D) \leq \left(\sum_{r=1}^{b-1}\binom{d_1}{d_{b-r+1}}\right)\exp\left(-\frac{(t-b)}{\overline{log}(b)H(D)}\right),
\end{align*}
where $H(D) = \underset{i\in [2:b]}{max}\frac{i}{\log(d_1)-\log(d_i)}$.
\end{theorem}
\begin{proof} We begin by noting that $P_e(D) = \sum_r P^r_e(D)$, where $P^r_e(D)$ is the probability that box~$1$ is eliminated in phase~$r$. Since at least one of the $r$ smallest communities is guaranteed to survive in phase~$r,$ box $1$ will not be eliminated in the $r$th phase if the agent has seen at least $d_{b-r+1}+1$ distinct samples from box~1. Thus, $P^r_e(D)$ is upper bounded by the probability of the event that there exists a subset of $d_1 - d_{b-r+1}$ individuals in box $1$, such that none of them are sampled in the $K_r$ queries made until the end of the $r$th phase. Therefore,
\begin{align*}
    P^r_e(D) &\leq \binom{d_1}{d_{b-r+1}} \left(1 - \frac{(d_1-d_{b-r+1})}{d_1}\right)^{K_r} 
    \\ \implies P^r_e(D) &\leq \binom{d_1}{d_{b-r+1}}\exp\left(-K_r\log\left(\frac{d_1}{d_{b-r+1}}\right)\right)
\end{align*}
Summing across $r$, we get that 
\begin{equation}
\label{Eqn:CC-SR}
    P_e(D) \leq \sum_{r=1}^{b-1}\binom{d_1}{d_{b-r+1}}\exp\left(-K_r\log\left(\frac{d_1}{d_{b-r+1}}\right)\right)
\end{equation}
Using $K_r = \ceil{\frac{1}{\overline{log}(b)}\frac{t-b}{b-r+1}}$ for $1 \leq r \leq b-1$, we note that 
\begin{align*}
    K_r\log\left(\frac{d_1}{d_{b-r+1}}\right) \geq \frac{(t-b)\log\left(\frac{d_1}{d_{b-r+1}}\right)}{\overline{log}(b)(b-r+1)} \geq \frac{(t-b)}{\overline{log}(b)H(D)}.
\end{align*}
\ignore{\nk{Shouldn't the first ineq be an equality and in the second term there should be a floor.} \jk{Have fixed this issue, I think. The algorithm description still has floors where their should be ceilings.}}
Combining with \eqref{Eqn:CC-SR}, we have 
\begin{align*}
    P_e(D) \leq \left(\sum_{r=1}^{b-1}\binom{d_1}{d_{b-r+1}}\right)\exp\left(-\frac{(t-b)}{\overline{log}(b)H(D)}\right) .
\end{align*}
%\jk{This proof needs more elaboration. The probability of error should be decomposed as the sum of the probabilities of elimination in round $r,$ with each such term being further bounded by a (reasoned) coupon collector argument. Also not sure about the pre-factor in the theorem statement.}
\end{proof}

Having analysed the CC-SR algorithm and the DS-SR algorithms, it is instructive to compare the exponential decay rates corresponding to the upper bounds of the probability of error under these algorithms. From Theorems~\ref{theorem:collision_audibert} and~\ref{thm:distinctsamplesAudibert}, this boils down to comparing the instance-dependent parameters $H^c(D)$ and $H(D)$ respectively, which encode the `hardness' of the underlying instance. Note that the values of these parameters are larger for instances where the size of the largest community is close to the sizes of the competing communities, and hence it would be harder for an algorithm to correctly estimate the mode. Consequently, the achievable probability of error from Theorems~\ref{theorem:collision_audibert} and~\ref{thm:distinctsamplesAudibert} is  also higher for harder instances.
%Comparing the performance of the Collision-based SR algorithm with the Distinct Samples SR algorithm in Theorems~\ref{theorem:collision_audibert} and \ref{thm:distinctsamplesAudibert} respectively, note that the rate of decay of the probability of error is governed by the instance-dependent parameters $H^c(D)$ and $H(D)$ respectively. We have
%\begin{align*}
%{H}^c(D) &= \underset{i\in\{2,...,b\}}{max}\frac{i\Delta_{i}^{-2}}{d_i}, \Delta_{i} = \frac{1}{d_i}-\frac{1}{d_1} \\
%    H(D) &= \underset{i\in\{2,...,b\}}{max}\frac{i}{\log(d_1)-\log(d_i)}
%\end{align*}
Furthermore, note that 
\begin{align*}
H^c(D) &= \underset{i\in [2:b]}{max} \frac{id_1^2d_i}{(d_1-d_i)^2} \stackrel{(a)}{>} \underset{i\in[2:b]}{max}\frac{d_1d_i}{d_1-d_i}  \frac{i}{\log(d_1)-\log(d_i)} \\ 
&\geq \frac{d_1d_b}{d_1-d_b} \underset{i\in[2:b]}{max} \frac{i}{\log(d_1)-\log(d_i)} =\frac{d_1d_b}{d_1-d_b}H(D).
\end{align*}
\ignore{\nk{I am not sure about the last line above..can you carefully check it..also the `max'' disappears in between}
\jk{Have added a step to make this clearer. Pl. re-check.}}

Here, the bound $(a)$ follows from the fact that $\log(x) > \frac{x-1}{x}$ for $x > 1.$ Since $H^c(D) > \frac{d_1d_b}{d_1-d_b}H(D),$ this means that $H^c(D) \gg H(D)$ for most instances of interest, which suggests that the DS-SR algorithm has a far superior performance as compared to the CC-SR algorithm (at least for large budget values). Our simulation results in Section~\ref{sec:experimental_results} are also consistent with this observation.
%Thus we have that $H^c(D)$ can be much larger than $H^c(D)$, and thus from Theorems~\ref{theorem:collision_audibert} and \ref{thm:distinctsamplesAudibert} we have that the error probability decay rate is much faster in case of the Distinct Samples SR algorithm. 

Next, we establish the near optimality of the Distinct Samples SR algorithm via an information theoretic lower bound.

\subsection{Lower Bounds}
While the decay rate in the upper bound of the DS-SR algorithm was expressed in terms of the hardness parameter $H(D),$ the information theoretic lower bound for the separated community setting is expressed in terms of a related hardness parameter $H_2(D) := \sum_{i=2}^{b}\frac{1}{\log(d_1)-\log(d_i)}.$ $H(D)$ and $H_2(D)$ are comparable upto a logarithmic (in the number of boxes) factor, as shown below.
\ignore{
We define another quantity $H_2(D)$ which is at most a ${\color{red}2\overline{log}(b)}$ factor away from the quantity $H(D)$ defined in Theorem~\ref{thm:distinctsamplesAudibert}, and state our lower bounds in terms of this. First, we prove this fact. \ignore{\sj{Converted theorem to lemma as discussed}}
}
\begin{lemma} \label{theorem:separatedupperlowerH} 
\ignore{Let us define $H(D) = \underset{i\in[2:b]}{max}\frac{i}{\log(d_1)-\log(d_i)}$ and $H_2(D) = \sum_{i=2}^{b}\frac{1}{\log(d_1)-\log(d_i)}$, where $d_i$ are ordered in a non-increasing fashion, and $d_1$ is the unique largest quantity: $d_1 > d_2 \geq d_3... \geq d_b$. Then, the following is true:
}
%\begin{align*}
$\frac{H(D)}{2} \leq H_2(D) \leq \overline{log}(b)H(D).$
%\end{align*}
%where $\overline{log}(b) = \frac{1}{2}+\sum_{i=2}^{b}\frac{1}{i}$. 
\end{lemma}
%\begin{proof} 
The proof of Lemma~\ref{theorem:separatedupperlowerH} can be found in~\ref{sec:other_lemmas}.
%\end{proof}

%\sj{Added a 1/2 factor to the LHS, because previously we had written $j$ instead of $j-1$ in the RHS. Doesn't change anything significant here. Doing the same in box community.}
%JK: This looks good.

We now state a lower bound on the probability of error in the separate community setting for any algorithm in a natural  algorithm class. The lower bound is non-asymptotic and is expressed in terms of the maximum of the probability of error under the original instance and an alternate instance which has a lower `hardness'. This is similar in form  to the corresponding lower bound for the standard multi-armed bandit setting in~\cite[Theorem 16]{Kaufmann16a}.  
\begin{theorem} \label{theorem:lower_bound_separated}
In the separated community setting, consider any algorithm that only uses the number of distinct samples from each community (box) to decide which box to sample from at each instant as well as to make the final estimate of the community mode. For any instance $D$, there exists an alternate instance $D^{[a]}, a \in [2:b]$, such that $H_2(D^{[a]}) \leq H_2(D)$ and 
\begin{align*}
    max\left(P_e(D),P_e(D^{[a]})\right) \geq \frac{1}{4}\exp\left(-\frac{3t}{H_2(D)}\right).
\end{align*}
In the alternate instance $D^{[a]}$, only the size of community $a$ is changed from $d_a$ to $\ceil{\frac{d_1^2}{d_a}}$. 
\end{theorem}
The proof of Theorem~\ref{theorem:lower_bound_separated} uses the following lemma.
\begin{lemma} \label{lemma:expected_samples_ub_separated}
For any algorithm $\mathcal{A}$ and instance $D$, there exists a box (community) $a \in [2:b]$ such that $E_D[N_a(t)] \leq \frac{t}{(\log(d_1)-\log(d_a))H_2(D)}$, where $N_a(t)$ denotes the number of times box $a$ is sampled in $t$ queries under $\mathcal{A}$.
\end{lemma}
\begin{proof} %We prove this by contradiction. 
Assume there exists no such community. Then,
\begin{align*}
%    &E_D[N_a(t)] > \frac{t}{(\log(d_1)-\log(d_a))H_2(D)} \forall a \in [2:b] \\
%    \Rightarrow &
\sum_{a=2}^{b}E_D[N_a(t)] > \sum_{a=2}^{b}\frac{t}{(\log(d_1)-\log(d_a))H_2(D)} = t,
 %   \Rightarrow t = \sum_{a=1}^{b}E[N_a(t)] \geq \sum_{a=2}^{b}E[N_a(t)] > t
\end{align*}
which is a contradiction.
%However this results in a contradiction, since we have $t = \sum_{a=1}^{b}E_D[N_a(t)] \ge \sum_{a=2}^{b}E_D[N_a(t)]$.
\end{proof}
\begin{proof}[Proof of Theorem~\ref{theorem:lower_bound_separated}]
Consider an algorithm $\mathcal{A}$ which bases all decisions only on the number of distinct individuals seen from each community (box). 
%This is similar to the proof of Theorem~\ref{theorem:mixed_identity_lb}. 
In this case, $S_j,$ the number of distinct samples from box (community)~$j$ evolves as a Markov chain over $[0:d_j],$ with transitions occurring each time the box is pulled. From state~$s,$ this chain transitions to (the same) state~$s$ with probability~$q_D^j(s,s) = \frac{s}{d_j},$ and to state~$s+1$ with probability $q_D^j(s,s) = \frac{d_j - s}{d_j}.$

\ignore{
However, in this case, to model the state evolution denoted by $\{X_j\}_{j=1}^{t}$,  we have one Markov chain corresponding to each box and when a particular box is sampled at a given instant, a transition is recorded in the corresponding Markov chain. The underlying state space is given by $\mathcal{S} = \mathcal{S}_1 \times \mathcal{S}_2 \times \ldots \times \mathcal{S}_b$, where $\mathcal{S}_j$ denotes the state space for the Markov chain corresponding to box $j$ which is given by $\mathcal{S}_j = \{0,1,\ldots,d_j\}$.  
%Note that the state at time $0$ is given by $X_0 = (0,0,\ldots,0)$.

Next, let us write down the transition probabilities $q_{D}^k(u,v)$ for each possible state transition $(u,v)$ in the chain corresponding to box $k$. Suppose we are in state $x^k \in \{0,1,\ldots,d_k\}$. Then after the next query where box $k$ is sampled, we can only go from state $x^k$ to itself or to state $x^k+1$. We have a self-transition at state $x^k$ if in the next query, the sampled individual has already been seen before, and thus the probability $q_{D}^k(x^k,x^k)$ of a self-transition at state $x^k$ is given by $x^k / d_k$. On the other hand, we have the probability $q_{D}^k(x^k,x^k +1)$ of going from $x^k$ to $x^k +1$ is given by $(d_k-x^k)^+/d_k$.
} %end ignore

Now, from Lemma~\ref{lemma:expected_samples_ub_separated} there exist a box $a \in [2:b]$ which satisfies $E[N_a(t)] \leq \frac{t}{(\log(d_1)-\log(d_a))H_2(D)}$. Consider the alternate instance $D^{[a]} = (d_1', d_2', \ldots, d_b')$ mentioned in the statement of the theorem, wherein $d_a' = \ceil{d_1^2 / d_a}$, $d_j' = d_j \ \forall j \neq a$. Note that the community mode under the alternate instance $D'$ is~$a,$ different from that under the original instance $D$. Furthermore, note that under the alternate instance $D^{[a]}$ the transition probabilities $q_{D^{[a]}}^k(u,v)$ remain the same for all $k \neq a$. For box~$a,$
\begin{align}
\label{Eqn:LLAlt1}
    \log\left(\frac{q_{D}^a(s,s)}{q_{D^{[a]}}^a(s,s)}\right) &= \log\left(\frac{\ceil{d_1^2/d_a}}{d_a}\right) \leq \log\left(\frac{d_1^3}{d_a^3}\right), \\ 
    \log\left(\frac{q_{D}^a(s,s+1)}{q_{D^{[a]}}^a(s,s+1)}\right) &= \log\left(\frac{1-s/d_a}{1-s/\ceil{d_1^2/d_a}}\right).
    \label{Eqn:LLAlt2}
\end{align}
Here, \eqref{Eqn:LLAlt1} because
\begin{gather*}
    \ceil{d_1^2/d_a} \leq 1+d_1^2/d_a = (d_a+d_1^2)/d_a \imp
    \frac{\ceil{d_1^2/d_a}}{d_a} \leq \frac{d_a+d_1^2}{d_a^2} = \frac{d_a^2+d_1^2d_a}{d_a^3} \leq \frac{d_1^3}{d_a^3}.
\end{gather*}

%%%%%%%%%%%%%%%%%%%%%%%%%%%%%%%%%%%%%%%%%%%%%%%%%%%%%%
\ignore{
For the Markov chain corresponding to box $a$, we have $q_{D^{[a]}}^a(x,x) = x / d_a'$ and $q_{D^{[a]}}^a(x,x +1) = {\color{red}(d_a'-x)/d_a'}$ \sj{removed the $^*$} for each $x \in \{0,1,\ldots,d_a\}$. Thus, for state $x$, if it exists under both $D$ and $D^{[a]}$, we have
}
\ignore{
\begin{gather}
% \label{Eqn:LLAlt1}
    \log\left(\frac{q_{D}^a(x,x)}{q_{D^{[a]}}^a(x,x)}\right) = \log\left(\frac{d_1^2}{d_a^2}\right) \\ 
    \log\left(\frac{q_{D}^a(x,x+1)}{q_{D^{[a]}}^a(x,x+1)}\right) = \color{red}\log\left(\frac{(d_a-x)d_1^2}{d_a(d_1^2-xd_a)}\right)
    % \label{Eqn:LLAlt2}
\end{gather}}
\ignore{
{\color{red}
Note that
\begin{gather*}
    \ceil{d_1^2/d_a} \leq 1+d_1^2/d_a = (d_a+d_1^2)/d_a \\
    \frac{\ceil{d_1^2/d_a}}{d_a} \leq \frac{d_a+d_1^2}{d_a^2} = \frac{d_a^2+d_1^2d_a}{d_a^3} \leq \frac{d_1^3}{d_a^3}
\end{gather*}}
}
%%%%%%%%%%%%%%%%%%%%%%%%%%%%%%%%%%%%%%%%%%%%%%%%%%%%%%

Next, let $\mathbb{P}_{D}, \mathbb{P}_{D^{[a]}}$ denote the probability measures induced by the algorithm under consideration by the instances $D,$ $D^{[a]},$ respectively. Then, given a trajectory $x = (a(1),s(1),\cdots,a(t),s(t)),$ where $a(k)$ denotes the box pulled on the $k$th query (action), and $s(k) = (s_j(k),\ j \in [b])$ is the vector of states corresponding to the arms after the $k$th query, the log-likelihood ratio is given by 
\begin{align*}
\log\frac{\mathbb{P}_{D}(x)}{\mathbb{P}_{D^{[a]}}(x)} = \sum_k \sum_{u,v} N_k(u,v,0,t)\log\left(\frac{q_{D}^k(u,v)}{q_{D^{[a]}}^k(u,v)}\right),
\end{align*}
where $N_k(u, v, 0, t)$ represents the number of times the transition from state $u$ to state $v$ happens in the Markov chain corresponding to box $k$ over the $t$ queries. Combining with \eqref{Eqn:LLAlt1}, \eqref{Eqn:LLAlt2}, we get
\begin{align*}
  D(\mathbb{P}_{D}||\mathbb{P}_{D^{[a]}}) &=   E_{D}\left[\log\frac{\mathbb{P}_{D}(x)}{\mathbb{P}_{D^{[a]}}(x)}\right] \\
  &\leq \sum_s E_{D}[N_a(s,s,0,t)]\log\left(\frac{d_1^3}{d_a^3}\right) + E_{D}[N_a(s,s+1,0,t)]\log\left(\frac{1-s/d_a}{1-s/\ceil{d_1^2/d_a}}\right)
\end{align*}
where $D(\cdot || \cdot)$ denotes the Kullback-Leibler divergence. Note that 
\begin{gather*}
    \ceil{d_1^2/d_a} > d_a \implies
    \frac{1-s/d_a}{1-s/\ceil{d_1^2/d_a}} \leq 1 \implies \log\left(\frac{1-s/d_a}{1-s/\ceil{d_1^2/d_a}}\right) \leq 0 .
\end{gather*}
% \sj{At some point, we should write that $\floor{d_1^2/d_a} > d_a$ means that both instances have different optimal communities}
Thus, we have 
\begin{gather*}
    D(\mathbb{P}_{D}||\mathbb{P}_{D^{[a]}}) \leq \sum_s E_{D}[N_a(s,s,0,t)]\log\left(\frac{d_1^3}{d_a^3}\right) \leq E_{D}[N_a(t)]\log\left(\frac{d_1^3}{d_a^3}\right)
\end{gather*}
\newtext{Next, we use Lemma~20 from~\cite{Kaufmann16a} (alternatively, see Lemma~\ref{lemma:mixed_identityless_lb_proof_3} in \ref{sec:other_lemmas}) to get that} 
\begin{align*}
     max\left(P_e(D),P_e(D^{[a]})\right) \geq \frac{1}{4}\exp\left(-D(\mathbb{P}_{D}||\mathbb{P}_{D^{[a]}})\right) \geq \frac{1}{4}\exp\left(-E_D[N_a(t)]\log\left(\frac{d_1^3}{d_a^3}\right)\right),
\end{align*}
where $P_e(D)$ is the probability of error under instance $D$. Finally, we use the bound on $E_D[N_a(t)]$ from Lemma~\ref{lemma:expected_samples_ub_separated} to get 
\begin{align*}
    \max\left(P_e\left(D\right),P_e(D^{[a]})\right) \geq \frac{1}{4}\exp\left(-\frac{3t}{H_2(D)}\right).
\end{align*}

It now remains to show that $H_2(D^{[a]}) \leq H_2(D)$. This is equivalent to showing
\begin{gather*}
    \sum_{i \in [b], i \neq a} \frac{1}{\log(\ceil{\frac{d_1^2}{d_a}})-\log(d_i)} \leq \sum_{i \in [b], i \neq 1} \frac{1}{\log(d_1)-\log(d_i)}.
\end{gather*}
This condition follows from the following term-by-term comparisons:
\begin{align*}
    \frac{1}{\log(\ceil{\frac{d_1^2}{d_a}})-\log(d_i)} &\leq \frac{1}{\log(d_1)-\log(d_i)}\quad (i \neq 1, a) \\
    \frac{1}{\log(\ceil{\frac{d_1^2}{d_a}})-\log(d_1)} &\leq \frac{1}{\log(d_1)-\log(d_a)}
\end{align*}
%Hence, $H_2(D^{[a]}) \leq H_2(D)$.
\end{proof}

Comparing the upper and lower bounds on the probability of error for the separated community setting in Theorems~\ref{thm:distinctsamplesAudibert} and \ref{theorem:lower_bound_separated}, we see that the expressions for the decay rates differ (ignoring universal constants) in terms of $H(D)$ vs $H_2(D)$, which from Lemma~\ref{theorem:separatedupperlowerH}, are at most a factor of $\overline{log}(b)$ apart. In other words, the decay rate under DS-SR is optimal, upto a logarithmic (in the number of boxes) factor. This is similar to the optimality guarantees available in fixed-budget MAB setting (see \cite{Aud10,Kaufmann16a}). 
\section{Community-disjoint Box Setting}
\label{sec:boxcomm}

In this section, we consider an intermediate setting that generalizes both the mixed and separated community settings. Specifically, we consider the case where each community exists in exactly one box; i.e, all the members of a community~$j$ are present in the same box. (Though any box may contain multiple communities.) In this setting, which we refer to as the \emph{community-disjoint box setting,} we propose algorithms that combine elements from the algorithms presented before for the mixed and separated community settings. For a class of reasonable instances, we are also able to establish the near optimality of certain algorithms. Finally, we show that the algorithms presented in this section can be generalized to handle the most general model, where communities are arbitrarily spread across boxes.

 Under the community-disjoint box setting, each column of the instance matrix $D$ has exactly one non-zero entry. Without loss of generality, we assume that $d_{11}$ is the largest value in the matrix~$D$; hence, box~1 contains the largest community (also labeled~1). Also without loss of generality, we order boxes by the sizes of the largest communities in them; i.e, if $g_i, 1 \leq i \leq b$ is the size of the largest community in box $i$, then $d_{11} = $ $g_1 > g_2 \geq g_3 \geq ... \geq g_b$. Additionally, we define $c_i$ to be the largest \emph{competing} community in a box--that is, $c_i = g_i, i \neq 1$, and $c_1$ is the second largest community in the first box. We state our results in terms of $d_{11}$ and $(c_i,\ i \in [b]).$
%\jk{What is known to the learning agent a priori?}

\subsection{Algorithms}
\label{sec:box_algos}

The first algorithm we consider for this setting is a generalization of the Distinct Samples SR algorithm from Algorithm~\ref{alg:distinctsamplesaudibert}, where we now eliminate boxes successively. Specifically, the algorithm proceeds in $b-1$ phases; one box being eliminated from subsequent consideration in each of the phases. At the end of the final phase, the algorithm outputs the community that produced the largest number of distinct samples from the last surviving box. Since we have multiple communities in each box, our elimination criterion in each phase is based on the seemingly largest community in each surviving box. In particular, let $S_{ij}^r$ denote the number of distinct individuals encountered from community~$j$ in box~$i$ at the end of phase~$r.$ We eliminate, at the end of phase~$r,$ the (surviving) box that minimizes~$\max_{j} S_{ij}^r.$ This algorithm, which we continue to refer to as the Distinct Samples SR (DS-SR) algorithm (with some abuse of notation), is presented formally in Algorithm~\ref{alg:boxsamplesaudibert}.

% end ignore

\begin{algorithm}[t]
\caption{Distinct Samples SR algorithm (community-disjoint box setting)}\label{alg:boxsamplesaudibert}
\begin{algorithmic}[1]
\State Set $\mathcal{B} = [b]$ \Comment{Set of surviving boxes}
\State Set $K_0 = 0$, $K_r=\ceil{\frac{1}{\overline{\log}(b)}\frac{t-b}{b-r+1}} \quad (1 \leq r \leq b-1$)
\For{$r = 1, 2, ..b-1$}
\State Sample each box in $\mathcal{B},$ $K_r - K_{r-1}$ times
\State Set $S_{ij}^r$ as number of distinct individuals seen so far from community~$j$ in box~$i \in \mathcal{B}$
\State Set, for $i \in \mathcal{B},$ $f_i = \max_{j} S_{ij}^r$ %\Comment{Defines rank of surviving boxes}
\State $\mathcal{B} = \mathcal{B} \setminus \{\argmin_{i \in \mathcal{B}} f_i\}$ \hspace{2 cm} (ties broken randomly) %\Comment{Eliminate lowest ranked box}
\EndFor
\State Set $\hat{b}$ as lone surviving box in $\mathcal{B}$
%\State Sample box $\hat{b},$ $K_b - K_{b-1}$ times 
%\State Set $S_{\hat{b}j}^b$ as number of distinct individuals seen so far from community~$j$ in box~$\hat{b}$
\State \textbf{Return} $\hat{h}^* = \argmax_{j}  S_{\hat{b}j}^{(b-1)}$ \hspace{2 cm} (ties broken randomly) 
\end{algorithmic}
\end{algorithm}

\begin{theorem}
\label{Thm:UBBox}
In the community-disjoint box setting, for any instance $D$, the Distinct Samples SR (DS-SR) algorithm given in Algorithm~\ref{alg:boxsamplesaudibert} has a probability of error upper bounded as
\begin{equation}
    P_e(D) \leq \left(\sum_{i=2}^{b} \binom{d_{11}}{c_i}\right) \exp\left(-\frac{(t-b)}{\overline{log}(b)H^b(D)}\right) + \binom{d_{11}}{c_1}\exp\left(-\frac{(t-b)\log\left(\frac{N_1}{N_1-d_{11}+c_1}\right)}{2\overline{log}(b)}\right) ,
    \label{eq:DS-SR_box_upper_bound}
\end{equation}
where $H^b(D) = \underset{i\in [2:b]}\max \frac{i}{\log(N_1) - \log(N_1-d_{11}+c_i)}$. 
\end{theorem}

\ignore{
\begin{theorem}
\label{Thm:UBBox}
For any instance $D$, the Distinct Samples SR (DS-SR) algorithm given in Algorithm~\ref{alg:boxsamplesaudibert} has a probability of error upper bounded as
\begin{gather*}
    P_e(D) \leq \left(\sum_i \binom{d_{11}}{c_i}\right) \exp\left(-\frac{(t-b-1)}{\overline{log}(b+1)H^b(D)}\right),
\end{gather*}
where $H^b(D) = \underset{i\in[b]}\max \frac{i+1}{\log(N_1) - \log(N_1-d_{11}+c_i)}$.
\end{theorem}}

The upper bound on the probability of error under the DS-SR algorithm above is a sum of two terms. The first term in~\eqref{eq:DS-SR_box_upper_bound} bounds the probability of misidentifying the box containing the largest community, while the second term in~\eqref{eq:DS-SR_box_upper_bound} bounds the probability of misidentifying the largest community within the correct box (box~1). Not surprisingly, the second term is structurally similar to the bound~\eqref{eq:community_coupon_collector} we obtained in Theorem~\ref{theorem:mixed_identity_ub} for the mixed community setting (restricted to box~1). The proof of Theorem~\ref{Thm:UBBox} can be found in~\ref{sec:proof_ub_box}. %\sj{I have added the proof environment here, because even when I state a proof is in the appendix, I still write it in this environment. Can be removed to save space later.}

%{\color{red}Note that both of the terms in the upper bound can be combined by extending the definition of $H^b(D)$, but we write it in this form in order to compare the performance with our lower bound.} 
The DS-SR algorithm works well in practice, particularly for large budget values. However, its performance can be sub-par for moderate budget values on certain types of instances; particularly instances where the largest community is contained within a very large box. In such cases, it can happen that $\Exp{S_{11}^r} < \Exp{S_{ij}^r}$ for another community~$j$ in a box~$i \neq 1,$ making it likely that box~1 gets eliminated early. We propose modified algorithms to resolve this issue, under the additional assumption that the box sizes are known a priori to the learning agent.\footnote{This is a natural assumption is several applications. For example, in the context of election polling, an agent might know a priori the total number of voters in each city/state.} The first modification replaces uniform exploration of boxes with a proportional exploration of the surviving boxes in each phase, resulting in a sampling process (within each phase) somewhat analogous to the mixed community setting considered in Section~\ref{sec:mixedcomm}. A second class of algorithms retains uniform box exploration, but normalizes $S_{ij}^r$ to reflect the size of each box (algorithms in this class differing with respect to the specific normalization performed). This latter class of algorithm can also be extended to the original setting where the box sizes are unknown, by replacing the box size by its maximum likelihood estimator.

We begin by describing our first modification of the DS-SR algorithm, which we refer to as the Distinct Samples Proportional SR (DS-PSR) algorithm. The DS-PSR algorithm apportions the budget across phases in the same manner as DS-SR, but the queries within each phase are distributed across surviving boxes in proportion to their sizes. Formally, this corresponds to the same description as Algorithm~\ref{alg:boxsamplesaudibert}, except that in Line~4, each box $i \in \mathcal{B}$ is sampled $T(\mathcal{B},r,i)$ times, where  $T(\mathcal{B},r,i) := \floor{\frac{N_i}{\sum_{k \in \mathcal{B}} N_k}(K_r-K_{r-1})(b-r+1)}.$ Experimentally, we find that DS-PSR performs very well. However, a tight characterization of the decay rate corresponding to the probability of error is challenging, since the number of queries available to each surviving box in phase~$r,$ for $1 < r \leq b-1,$ is a random quantity, that depends on the sequence of prior box eliminations.

Next, we describe the normalized variants of the DS-SR algorithm. The first, which we refer to as the Normalized Distinct Samples SR (NDS-SR) algorithm, is described by changing the definition of~$f_i$ in Line~6 of Algorithm~\ref{alg:boxsamplesaudibert} to \newtext{$$f_i^{\mathrm{NDS-SR}} = \max_j \frac{S_{ij}^r}{S_i^r} N_i,$$}  where $S_i^r$ denotes the number of distinct individuals seen from box~$i$ (across different communities) by the end of phase~$r.$ This normalization is justified as follows: $S_{ij}^r / S_i^r$ is an unbiased estimator of $d_{ij}/N_i,$ i.e., the fraction of box~$i$ that is comprised by community~$j.$ 

%Q: Why is $Z = \frac{S_{ij}^r}{S_i^r}$ an unbiased estimator of $d_{ij}/N_i$?

%A: Condition on $S_i^r = k.$ Then $$Z \eqdist \frac{\sum_{l = 1}^k I_l} {k},$$ where $I_l$ is a Bernoulli indicator of the individual $l$ being in community~$j.$ Thus, $\Exp{Z|S_i^r = k} = d_{ij}/N_i,$ meaning $\Exp{Z} = d_{ij}/N_i.$

The final variant we propose, referred to as the Expectation-Normalized Distinct Samples SR (ENDS-SR) algorithm, uses the following alternative normalization of $f_i$ in Line~6 of Algorithm~\ref{alg:boxsamplesaudibert}: \newtext{$$f_i^{\mathrm{ENDS-SR}} = \max_j \frac{S_{ij}^r}{\Exp{S_i^r}} N_i.$$} This normalization has a similar justification: indeed, $\frac{S_{ij}^r}{\Exp{S_i^r}}$ is another (more tractable) unbiased estimator of $d_{ij}/N_i.$ 

Both NDS-SR and ENDS-SR perform quite well in practice. It is challenging to analytically bound the performance of NDS-SR, due to the difficulty in concentrating the fractions $S_{ij}^r / S_i^r.$ However, the probability of error under ENDS-SR admits an upper bound analogous to that under DS-SR (albeit more cumbersome). Interestingly, the exponential decay rate of the probability of error under ENDS-SR is identical to that under DS-SR.
\begin{theorem}
\label{Thm:ENDS-SR}
In the community-disjoint box setting, for any instance $D$, 
\begin{equation*}
    \limsup_{t \ra \infty} \frac{\log P_e(D,\text{\emph{ENDS-SR}},t)}{t} \leq - 
    \frac{1}{\overline{log}(b)} \min \left(\frac{1}{H^b(D)},\frac{1}{2} \log\left(\frac{N_1}{N_1-d_{11}+c_1}\right) \right).
\end{equation*}
\end{theorem}
The proof of Theorem~\ref{Thm:ENDS-SR} can be found in~\ref{sec:NEDS_decay_rate}. The intuition behind Theorem~\ref{Thm:ENDS-SR} is that for large~$t,$ $\Exp{S_i^r} \approx N_i,$ so that $f_i^{\mathrm{ENDS-SR}} \approx S_{ij}^r,$ making the elimination criterion under ENDS-SR nearly identical to that under DS-SR.
\ignore{
\begin{proof} Let $P^i_e(D)$ denote the probability of the community mode being eliminated at the $i$th step; i.e, for $i \leq b-1, P^i_e(D)$ denotes the probability of removing box $1$ in phase~$i$ of SR, and $P^b_e(D)$ denotes the probability of choosing the wrong community from box~1 after this box survived the~$(b-1)$ SR phases. Then, we have
\begin{align*}
    P_e(D) &= \sum_{i=1}^{b-1} P^i_e(D) + P^b_e(D),\\
    P^i_e(D) &\leq \binom{d_{11}}{c_{b-i+1}}\exp\left(-K_i\log\left(\frac{N_1}{N_1-d_{11}+c_{b-i+1}}\right)\right) \quad (1 \leq i \leq b-1),\\
    P^b_e(D) &\leq \binom{d_{11}}{c_{1}}\exp\left(-K_{b-1}\log\left(\frac{N_1}{N_1-d_{11}+c_{1}}\right)\right),
\end{align*}
where the second and third statements are based on a coupon collector argument, similar to the one employed in the proof of Theorem~\ref{thm:distinctsamplesAudibert} for the separated community setting. The proof is now completed by substituting the values of $K_r,$ and using the definition of $H^b(D).$ 
%
%Therefore, we get
%\begin{gather*}
%    P_e(D) \leq \max\left(\left(\sum_{i=2}^{b} \binom{d_{11}}{c_i}\right) \exp\left(-\frac{(t-b-1)}{\overline{log}(b+1)H^b(D)}\right), \binom{d_{11}}{c_1}\exp\left(-\frac{(t-b-1)\log\left(\frac{N_1}{N_1-d_{11}+c_1}\right)}{2\overline{log}(b+1)}\right) \right)
%\end{gather*}
\end{proof}}

\subsection{Lower Bounds}
We now derive information theoretic lower bounds on the probability of error in the community-disjoint box setting, and compare the decay rates suggested by the lower bounds to the decay rate under DS-SR.

Our first lower bound captures the complexity of simply identifying the largest community from within box~1.
\begin{theorem}
\label{Thm:BoxLB-MixedComm}
For any consistent algorithm, the probability of error corresponding to an instance~$D$ in the community-disjoint box setting is asymptotically bounded below as
$$\liminf_{t \ra \infty} \frac{P_e(D)}{t} \geq - \log\left(\frac{N_1}{N_1 - (d_{11} - c_1 + 1)} \right).$$
\end{theorem}
Note that Theorem~\ref{Thm:BoxLB-MixedComm} follows directly from Theorem~\ref{theorem:mixed_identity_ub} for the mixed community setting.

Our second lower bound is complementary, in that it captures the complexity of identifying the box containing the largest community. To state this bound, we define 
%$H^b(D) = \underset{i\in [2:b]}{\max}\frac{i+1}{\log(N_1)-\log(N_1-d_{11}+c_i)}$ and 
$H^b_2(D) = \sum_{i=2}^{b}\frac{1}{\log(N_1)-\log(N_1-d_{11}+c_i)}.$\ignore{, where $c_2, c_3, ... c_b$ are ordered in a non-increasing fashion, while $c_1$ remains unconstrained in terms of the ordering} Then, following along similar lines as the proof of Theorem~\ref{thm:distinctsamplesAudibert}, we can show that
\begin{align*}
    \frac{H^b(D)}{2} \leq H^{b}_{2}(D) \leq \overline{log}(b)H^b(D).
\end{align*}

\begin{theorem}
\label{Thm:LBBox2}
In the community-disjoint box setting, consider any algorithm that only uses the number of distinct samples from each community to decide which box to sample from at each instant as well as to make the final estimate for the community mode. For any instance $D$, there exists an alternate instance $D^{[a]},\ a \in [2:b]$, with $H_2^b(D^{[a]}) \leq H_2^b(D)$ such that
\begin{align*}
    \max\left(P_e(D), P_e(D^{[a]})\right) \geq \frac{1}{4}\exp\left(-\frac{t\Gamma}{H_2^b(D)} \right),
\end{align*}
where \ignore{$\Gamma = \frac{\log(\ceil{\frac{N_1(N_a+d_{11})}{N_1-d_{11}+c_b}})-\log(N_a)}{\log(N_1)-\log(N_1-d_{11}+c_a)}.$}  \newtext{$\Gamma = \max\left(\frac{\log\left(\ceil{\frac{N_1(N_a-c_a+d_{11})}{(N_1-d_{11}+c_a)}}\right) - \log\left(N_a\right)}{\log\left(\frac{N_1}{N_1-d_{11}+c_a}\right)},  \max_{i=2}^{b} \frac{\log\left(\ceil{\frac{N_1(N_a-c_a+c_i)}{(N_1-d_{11}+c_i)}}\right)-\log\left(N_a\right)}{\log\left(\frac{N_1}{N_1-d_{11}+c_a}\right)} \right).$}
%
%\ignore{\Gamma = \frac{\log(N_1)+\log(N_a+d_{11})-\log(N_a)-\log(N_1-d_{11})}{\log(N_1)-\log(N_1-d_{11}+g_a)}}$. 
The alternate instance $D^{[a]}$ is constructed by increasing the size of only the largest community in box $a$\ignore{ from $g_a$ to $g_a' = N_a'-N_a+g_a$}, such that the new size of box $a$ is \ignore{$N_a' = \ceil{\frac{N_1(N_a+d_{11})}{N_1-d_{11}+c_b}}.$}
\newtext{$N_a' = \max\left(\ceil{N_1\frac{(N_a-c_a+d_{11})}{(N_1-d_{11}+c_a)}}, \max_{i=2}^{b} \ceil{N_1\frac{(N_a-c_a+c_i)}{(N_1-d_{11}+c_i)}} \right).$}

\end{theorem}
\ignore{\jk{Removed the subscript~$b$ in $D_b^{[a]}$.}
\sj{Have rewritten this; however, I'm not confident in where we should be writing $g$ and where we should be writing $c$. Please have a look - the only place where they are not replaceable is when defining $g'$.}}

The proof of Theorem~\ref{Thm:LBBox2} follows along similar lines as the proof of Theorem \ref{theorem:lower_bound_separated}. Details can be found in~\ref{sec:proof_lb_box}.

Comparing the upper and lower bounds on the probability of error for the box setting in Theorems~\ref{Thm:UBBox},~\ref{Thm:BoxLB-MixedComm}, and~\ref{Thm:LBBox2}, we see that the expressions for the exponents differ primarily in i) the presence of $H^b(D)$ vs $H_2^b(D)$, which differ by at most a factor of $\overline{log}(b)$; and ii) the presence of an additional factor $\Gamma$ in the lower bound. Note that 
\newtext{
\begin{gather*}
    \max\left(\ceil{N_1\frac{(N_a-c_a+d_{11})}{(N_1-d_{11}+c_a)}}, \max_{i=2}^{b} \ceil{N_1\frac{(N_a-c_a+c_i)}{(N_1-d_{11}+c_i)}} \right) \leq \ceil{\frac{N_1(N_a-c_a+d_{11})}{N_1-d_{11}+c_b}} \\ \leq \frac{N_1(N_a-c_a+d_{11})}{N_1-d_{11}+c_b} + 1 \leq \frac{N_1(N_a-c_a+d_{11}+1)}{N_1-d_{11}+c_b}.
\end{gather*}
Using $\frac{x-1}{x} \le \log(x) \leq x-1$ for all $x > 0$ and the above inequality, we get
\begin{gather*}
    \Gamma \leq \frac{\log(N_1)+\log(N_a-c_a+d_{11}+1)-\log(N_a)-\log(N_1-d_{11}+c_b)}{\log(N_1)-\log(N_1-d_{11}+c_a)} \\ = \frac{\log(N_1 / (N_1 - d_{11} + c_b))+\log((N_a-c_a+d_{11}+1)/N_a)}{\log(N_1 / (N_1-d_{11}+c_a))} \\
     \leq \frac{ (d_{11}-c_b) / (N_1 - d_{11} + c_b) + (d_{11}-c_a+1)/ N_a }{(d_{11} - c_a)/N_1}
    %  \\ \leq \frac{(d_{11} - c_b)}{(d_{11} - c_a)} \cdot \frac{N_1}{(N_1 - d_{11} + c_b)} \cdot \frac{(N_1 + N_a - d_{11}+\frac{N_1}{d_{11}}-1)}{N_a}
     \\
     \leq \frac{(d_{11} - c_b)}{(d_{11} - c_a)} \cdot \frac{N_1}{(N_1 - d_{11} + c_b)} \cdot \frac{(2N_1 + N_a)}{N_a}.
\end{gather*}
}
\ignore{
\begin{gather*}
    \ceil{\frac{N_1(N_a+d_{11})}{N_1-d_{11}+c_b}} \leq \frac{N_1(N_a+d_{11})}{N_1-d_{11}+c_b} + 1 \leq \frac{N_1(N_a+d_{11}+1)}{N_1-d_{11}+c_b}.
\end{gather*}
Using $\frac{x-1}{x} \le \log(x) \leq x-1$ for all $x > 0$ and the above inequality, we get
\begin{gather*}
    \Gamma \leq \frac{\log(N_1)+\log(N_a+d_{11}+1)-\log(N_a)-\log(N_1-d_{11}+c_b)}{\log(N_1)-\log(N_1-d_{11}+c_a)} \\ = \frac{\log(N_1 / (N_1 - d_{11} + c_b))+\log((N_a+d_{11}+1)/N_a)}{\log(N_1 / (N_1-d_{11}+c_a))} \\
     \leq \frac{ (d_{11}-c_b) / (N_1 - d_{11} + c_b) + (d_{11}+1)/ N_a }{(d_{11} - c_a)/N_1}
     = \frac{(d_{11} - c_b)}{(d_{11} - c_a)} \cdot \frac{N_1}{(N_1 - d_{11} + c_b)} \cdot \frac{(N_1 + N_a - d_{11}+\frac{N_1}{d_{11}}-1)}{N_a} \\
     \leq \frac{(d_{11} - c_b)}{(d_{11} - c_a)} \cdot \frac{N_1}{(N_1 - d_{11} + c_b)} \cdot \frac{(2N_1 + N_a)}{N_a}
\end{gather*}}
\ignore{
\begin{align*}
     \Gamma = \frac{\log(N_1)+\log(N_a+d_{11})-\log(N_a)-\log(N_1-d_{11})}{\log(N_1)-\log(N_1-d_{11}+c_a)} &= \frac{\log(N_1 / (N_1 - d_{11}))+\log((N_a+d_{11})/N_a)}{\log(N_1 / (N_1-d_{11}+c_a))} \\
     &\le \frac{ d_{11} / (N_1 - d_{11}) + d_{11}/ N_a }{(d_{11} - c_a)/N_1}\\
     &= \frac{d_{11}}{(d_{11} - c_a)} \cdot \frac{N_1}{(N_1 - d_{11})} \cdot \frac{(N_1 + N_a - d_{11})}{N_a}
     \end{align*} 
     \remove{
     \log(N_1)+\log(N_a+d_{11})-\log(N_a)-\log(N_1-d_{11}) \leq \frac{d_{11}(N_1+N_a)}{N_a(N_1-d_{11})} \\
     \log(N_1)-\log(N_1-d_{11}+c_a) \geq \frac{d_{11}-c_a}{N_1} \\
     \implies k \leq \frac{d_{11}(N_1+N_a)N_1}{(d_{11}-c_a)N_a(N_1-d_{11})}
     }}
In particular, the above inequality implies that $\Gamma$ is bounded by a constant under the following natural assumptions on the class of underlying instances: i) the largest community size is at most a fraction of its corresponding box size, i.e., $d_{11} \le (1 - \delta_1)N_1$ for some $\delta_1 > 0$; ii) the size of the competing communities in other boxes \ignore{the second largest community size }is most a fraction of the largest community size, i.e., $c_a \le (1 - \delta_2)d_{11}$ for some $\delta_2 > 0$ $\forall a \neq 1$; and iii) all the box sizes are within a multiplicative constant factor $\beta$ of each other $(\beta > 1)$.  Under these assumptions, $\Gamma \leq \frac{2\beta+1}{\delta_1\delta_2}$.
%\sj{Thought about it for a bit; I don't think there's anything more natural and weaker than what we have described here already.}

We compare this lower bound to the first term in the upper bound given in Theorem~\ref{Thm:UBBox}. We note that these terms only differ by an order of $\overline{log}(b)\Gamma$. When $\Gamma$ is bounded from above, such as in the case described above, the DS-SR estimator matches the lower bound upto logarithmic factors for the problem of picking the correct box in the final stage of the algorithm, and is hence near-optimal. Comparing the second term in the upper bound from Theorem~\ref{Thm:UBBox} to Theorem~\ref{Thm:BoxLB-MixedComm}, we find a similar logarithmic factor between the decay rates.
%multiplicative factor of $c.\overline{log}(b)$, where $c$ is a constant. 
Thus, the DS-SR algorithm is decay rate optimal up to logarithmic factors for the problem of picking the right community out of a box, given the correct box. This is natural and intuitive, due to its similarity with the mixed community DSM algorithm. Hence, the set of instances where DS-SR might not perform well in comparison to other algorithms can be characterized as instances where it is hard to pick the correct box containing the largest community; intuitively, these instances would produce a large value of the parameter~$\Gamma.$
%\sj{The above is a good point which can be reinforced in the Brazil dataset section in the numerics. We can also explain in the next section that the normalization helps DS-SR pick the correct box; it doesn't affect it's optimal performance in the single box setting at all.}

\ignore{
{\color{red}Under such a setting, if the first term in Theorem~\ref{Thm:UBBox} dominates, then the exponents in the lower and upper bounds for the probability of error in Theorems~\ref{Thm:UBBox} and \ref{Thm:LBBox2} differ by a multiplicative factor of at most $c . \overline{log}(b)$ for $c = \frac{2\beta+1}{\delta_1\delta_2}$.} \ignore{\jk{Can we express $c$ in terms of $\delta_1$ and $\delta_2$?}} \sj{Need to rethink this constant factor.}
{\color{red} If in fact the second term in the upper bound in Theorem~\ref{Thm:UBBox} dominates, then the mixed community lower bound for community $1$ gives us a matching lower bound upto a factor of $c.\overline{log}(b)$.}}

\subsection{The general setting}
\label{sec:general}
Finally, we consider the most general setting, where communities are arbitrarily spread across boxes. From an algorithmic standpoint, the key challenge here is that it is no longer appropriate to eliminate boxes from consideration sequentially as in SR algorithms, since the largest community might be spread across multiple boxes. Accordingly, the algorithms we propose for the general setting are `single phase' variants of the algorithms proposed in Section~\ref{sec:box_algos}.

The single phase variant of Algorithm~\ref{alg:boxsamplesaudibert}, which we refer to as the Distinct Samples Uniform Exploration (DS-UE) algorithm is stated as follows: sample each box~$\floor{t/b}$ times, and return the community that produces the largest number of distinct individuals. The probability of error under this algorithm can be bounded using the ideas we have used before, only the bounds are more cumbersome. 

If the box sizes are known, one can also perform a single-phase proportional sampling of boxes, resulting effectively in a sampling process similar to the mixed community setting (except the budget is apportioned deterministically across boxes rather than the random allocation in the mixed community setting) . We refer to the corresponding algorithm, which outputs the community that produced the largest number of distinct individuals after $t$ queries, as the Distinct Samples Proportional Exploration (DS-PE) algorithm. %The probability of error under DS-PE can also be bounded easily; see Appendix~\ref{??}.
 
Finally, we state the normalized single phase variant of DS-UE, which we refer to as NDS-UE: Each box is sampled $\floor{t/b}$ times, and the output of NDS-UE is the community that maximizes $\sum_{i} \frac{S_{ij}}{S_i} N_i.$ ENDS-UE can be analogously defined.
%(and analysed to have the same asymptotic decay rate as DS-UE). 
%\sj{Let's remove all the claims of having same decay rate if we can't add this in the Appendix, I think} JK: I agree.

\newtext{To summarize, some of our algorithms for the disjoint box setting can indeed be applied and evaluated analytically in the general setting. However, we do not at present have a tight information theoretic lower bound for the general setting (or indeed, even for the disjoint box setting); the proof techinques we have used in the lower bounds for the mixed/separated community settings appear to be insufficient to handle the general case. So even though our algorithms for the general setting perform well in empirical evaluations (see Section~\ref{sec:experimental_results}), new methodological innovations are required to close the gap between upper and lower bounds.}

\section{Experimental Results}
\label{sec:experimental_results}
In this section, we present extensive simulation results comparing the performance of various algorithms discussed in the previous sections. We use both synthetic data as well as data gathered from real-world datasets for our experiments. For each experiment, we averaged the results over multiple runs (500-3000 depending on the complexity of the instance).
%
%We perform monte carlo simulations and report probabilities of error for multiple algorithms under the different settings in %this paper. We use both synthetic data as well as data gathered from real datasets to perform our experiments on.
%
\subsection{Mixed Community Mode Estimation}
We begin with the mixed community setting studied in Section~\ref{sec:mixedcomm} where all individuals are placed in a single  box. We demonstrate the difference in performance of the identity-less Sample Frequency Maximization (SFM) and the identity-based Distinct Samples Maximization (DSM) algorithms via simulations on synthetic data. We consider two instances, each with $4000$ individuals in a single box, partitioned into communities as $[1000, 990, 600, 500, 500, 410]$ and $[1000, 900, 630, 520, 520, 430]$ respectively. As suggested by Theorems~\ref{theorem:mixed_identityless_lb_1} and \ref{theorem:mixed_identity_lb}, we find that the difference in the convergence rates of the two estimators becomes more pronounced when the two largest communities are close in size. See Figure~\ref{Fig:Mixed} where we plot the probability of error $\log(P_e)$ vs the query budget $t$ for the two instances.
\ignore{
We consider two examples; in Figure \ref{figure:hard_mixed_synthetic}, we consider a case where all the communities are close in size, and in Figure \ref{figure:easy_mixed_synthetic} we consider a case where there is a significant difference in the community sizes.

In this section using a synthetic dataset we created we illustrate empirically how having identity information can speed up the convergence of the algorithm.

Empirical Results also show that the difference in convergence rate when identity information is present vs when it is not present become much more pronounced as the size of the largest and the second largest community become closer.

We experiment with two algorithms in this setting, when identity information is present (or absent) then we use the algorithm described in section 3.1  (or section 3.2) resp).

\sg{Here maybe these two plots can illustrate that difference between the estimators is more important when communities sizes are closer, as expected theoretically}

\nk{I think we should select one instance where the top two communities are closer and then the second instance where the individuals are rebalanced so that the gap between top two is increased. These curves can even be on the same figure I think.}
}
\begin{figure}[t]
  \centering
  \begin{subfigure}[b]{0.45\textwidth}
  \centering 
  \includegraphics[width=\textwidth]{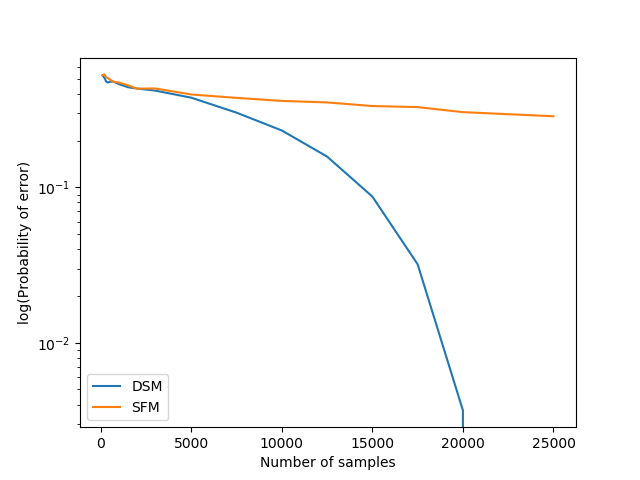}
  \caption{Instance: [1000, 990, 600, 500, 500, 410]}
  \end{subfigure}
    \hfill 
    \begin{subfigure}[b]{0.45\textwidth}
  \centering 
  \includegraphics[width=\textwidth]{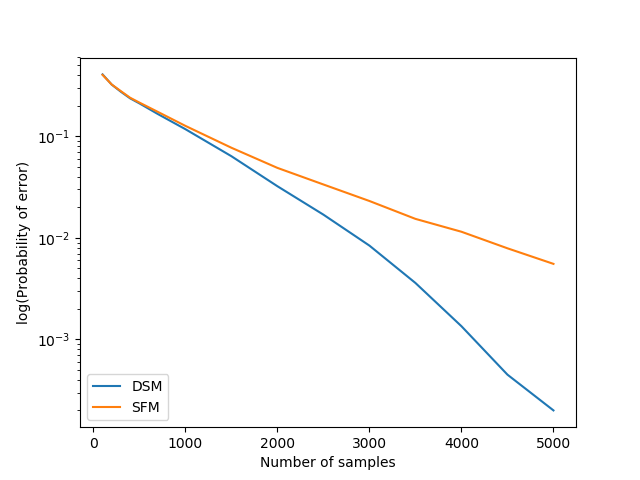}
  \caption{Instance: [1000, 900, 630, 520, 520, 430]}
  \end{subfigure}
  \caption{$\log(P_e(D))$ vs $t$ for mixed community setting}
  \label{Fig:Mixed}
\end{figure}

\ignore{
\begin{figure}[H]
	\centering
	\includegraphics[scale=0.5]{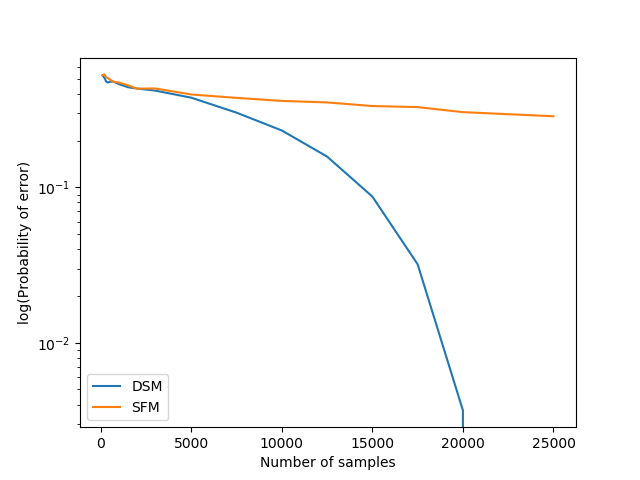}
	\caption{$P_e(D)$ vs t for community size distribution [1000, 990, 600, 500, 500, 410]}
	\label{figure:hard_mixed_synthetic}
	\centering
\end{figure}
\begin{figure}[H]
	\centering
	\includegraphics[scale=0.5]{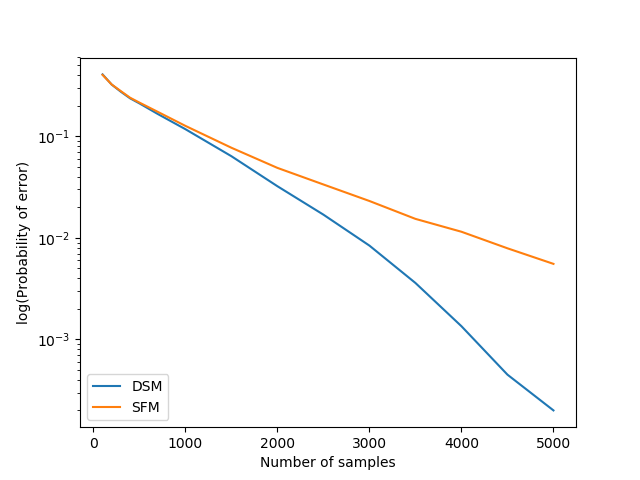}
	\caption{$P_e(D)$ vs t for community size distribution [1000, 900, 630, 520, 520, 430]}
	\label{figure:easy_mixed_synthetic}
	\centering
\end{figure}}

\subsection{Separated Community Mode Estimation}
Next, we consider the separated community setting studied in Section~\ref{sec:separatedcomm} where each community is in a unique box. As above, we consider two instances with community sizes given by $[1000, 990, 600, 500, 500, 410]$ and $[1000, 900, 630, 520, 520, 430]$ respectively. We plot the performance of the Consecutive-Collision SR (CC-SR) and Distinct Samples SR (DS-SR) algorithms in Figure~\ref{Fig:Sep}. As indicated by our results in Theorems~\ref{theorem:collision_audibert} and \ref{thm:distinctsamplesAudibert}, the DS-SR algorithm greatly outperforms the CC-SR algorithm.
%
%In this section, we reinforce our point from Section \ref{sec:separatedcomm}, stating that the Collision-based SR algorithm performs much worse compared to the Distinct samples SR algorithm. We can note this difference in the same cases we considered in mixed community, but with the communities spread into different boxes.
%
\begin{figure}[t]
  \centering
  \begin{subfigure}[b]{0.45\textwidth}
  \centering 
  \includegraphics[width=\textwidth]{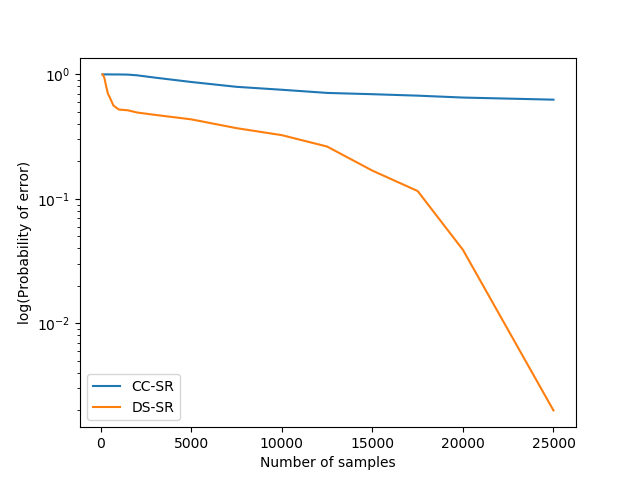}
  \caption{Instance: [1000, 990, 600, 500, 500, 410]}
  \end{subfigure}
    \hfill 
    \begin{subfigure}[b]{0.45\textwidth}
  \centering 
  \includegraphics[width=\textwidth]{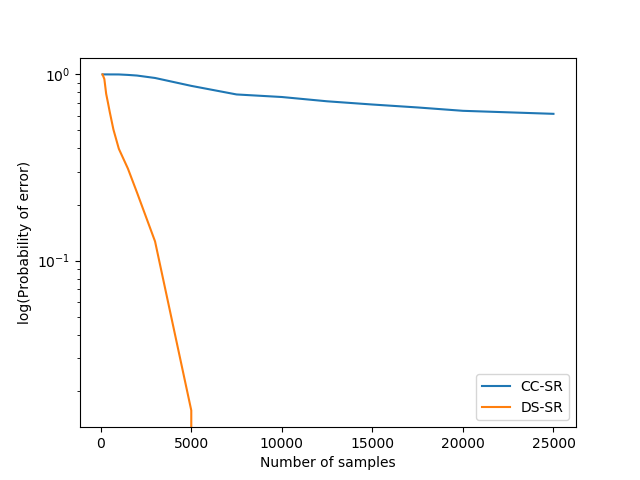}
  \caption{Instance: [1000, 900, 630, 520, 520, 430]}
  \end{subfigure}
  \caption{$\log(P_e(D))$ vs $t$ for separated community setting}
  \label{Fig:Sep}
\end{figure}

\subsection{Community-Disjoint Box Mode Estimation}
Here, we look at the setting where the communities are partitioned across the boxes and thus each box can have multiple communities, as described in Section~\ref{sec:boxcomm}. We use the following two real-world datasets for comparing the performance of various estimators under this setting.
\begin{itemize}
    \item Brazil Real Estate Dataset \cite{brasil_dataset}: This dataset contains  a total of 97353 apartment listings spread across 26 states and 3273 municipalities in Brazil. Mapping it to our framework, the apartments correspond to individual entities, the municipalities represent communities and the states they are located in denote the boxes. Our goal is to identify the municipality (community) with the largest number of listings by (randomly) sampling apartment listings from various states.  
    
    Corresponding to this dataset, the four largest communities (municipalities with the most listed apartments) are of sizes [3929, 2322, 2414, 1876]. The top five box sizes are [80935, 3551, 2035, 1871, 1646], with the largest box corresponding to the state of Sao Paolo. Thus, one box has a much larger size than all others in this dataset and in fact, contains all of the the four largest communities.
    
    \item Airbnb Rental Listing Dataset \cite{airbnb_dataset}:  This dataset contains a total of 48895 rental listings spread across $5$ regions and $221$ neighborhoods in New York city. Here, the apartments correspond to individual entities, the neighbourhoods represent communities and the broad regions they are located in denote the boxes.
    
    The top five communities (neighbourhoods) have sizes [3920, 3741, 2658, 2465, 1971]. The top 5 box sizes are [21661, 20104, 5666, 1091, 373]. Unlike the previous dataset, the two largest boxes (corresponding to Manhattan and Brooklyn respectively) are of comparable size here. Furthermore, the two boxes contain multiple competing communities of size comparable to the largest community. The largest box contains the communities with sizes 2658 and 1971, while the second largest box contains communities of sizes 3920 (mode), 3714, and 2465. 
\end{itemize}

\noindent{\bf Results}
We compare the performance of the various algorithms discussed in Section~\ref{sec:box_algos} on the two datasets described above. These include the Distinct Samples-Successive Rejects (DS-SR) and its generalization Distinct Samples Proportional SR (DS-PSR) when the box sizes are known. We also consider the normalized variants of DS-SR, given by Normalized Distinct Samples SR (NDS-SR) and Expectation-Normalized Distinct Samples SR (ENDS-SR) when box sizes are known as well as Normalized Distinct Samples SR (NDS-SR (MLE)) when the box sizes are unknown, by replacing the box size by its maximum likelihood estimator.

Figure~\ref{Fig:brazil} shows the performance of the various algorithms on the Brazil Real Estate dataset. DS-SR which splits queries uniformly across all surviving boxes performs the worst while DS-PSR which does the division in proportion to box sizes performs the best. This is to be expected since there is one box which is much larger than all others and this box contains all of the competing largest communities. Thus, because of the uniform exploration in DS-SR, there might be fewer samples from the individual communities in the largest box in the initial rounds and it might get eliminated, which explains the poor performance for moderate query budgets. This shortcoming is addressed by DS-PSR which assign many more queries to the largest box which contains the community mode. The normalized variants NDS-SR and ENDS-SR also perform much better than DS-SR since they use the box sizes to determine the elimination criteria in each round. In comparison to these, the NDS-SR (MLE)  performs poorer for low query budget due to erroneous box size estimates but demonstrates similar performance for larger budgets. 

\ignore{
The Brazil Real Estate Dataset contains one box which is much larger than the others, and this box contains all of the competing largest communities. This is an ideal case for algorithms which eliminate boxes quickly like Successive Rejects, or algorithms which apportion samples keeping box size in mind. This is observed in the plot below.
In the Brazil Real Estate Dataset both A4 and A7 which allocate samples in proportional to the box size perform well. \\
A3 which does that same but uses MLE of the box size to allocate budget instead of the actual box size does not perform as well for smaller amount of samples, presumably due to wrong initial guesses of MLE, but does significantly better than other competing algorithms after 10000 samples. \\
Also note that A8 which uses Audibert Sampling with Distinct Samples Frequency Estimator shows a flat curve, while both A2 which uses Audibert Sampling with Box Fraction Estimator and A5 which is Single Round with Uniform Sampling and Distinct Samples Frequency Estimator show decent performance, this is because when using the Distinct Samples Frequency Estimator the communities in smaller boxes get an advantage and the largest community (which is present in a very large box) gets eliminated early into the algorithm giving a very high probability of error. Thus, in this case normalizing with box size gives an advantage.
}
\begin{figure}[t]
  \centering
  \begin{subfigure}[b]{0.45\textwidth}
  \centering 
  \includegraphics[width=\textwidth]{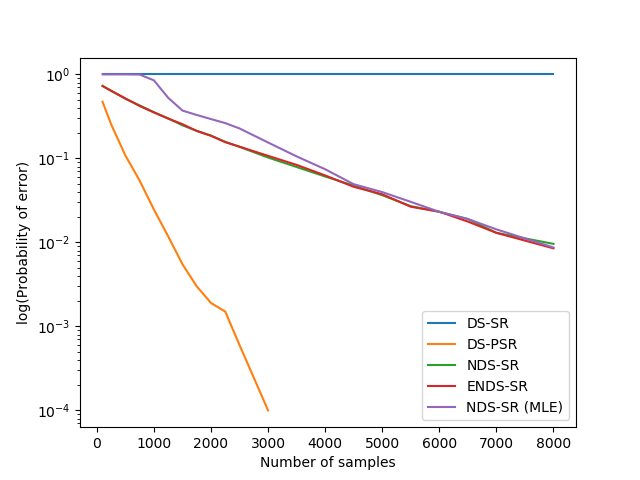}
  \caption{Brazil Real Estate Dataset}
  \label{Fig:brazil}
  \end{subfigure}
    \hfill 
    \begin{subfigure}[b]{0.45\textwidth}
  \centering 
  \includegraphics[width=\textwidth]{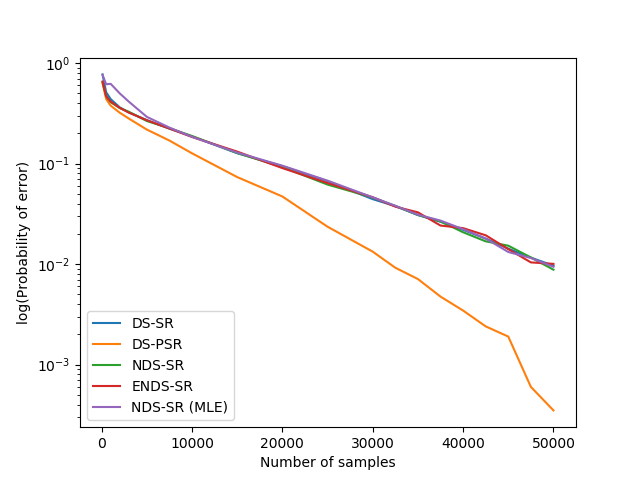}
  \caption{Airbnb Rental Listing Dataset}
    \label{Fig:Airbnb}
  \end{subfigure}
  \caption{$\log(P_e(D))$ vs $t$ for box community setting}
\end{figure}
\ignore{
\begin{figure}[H]
	\centering
	\includegraphics[scale=0.5]{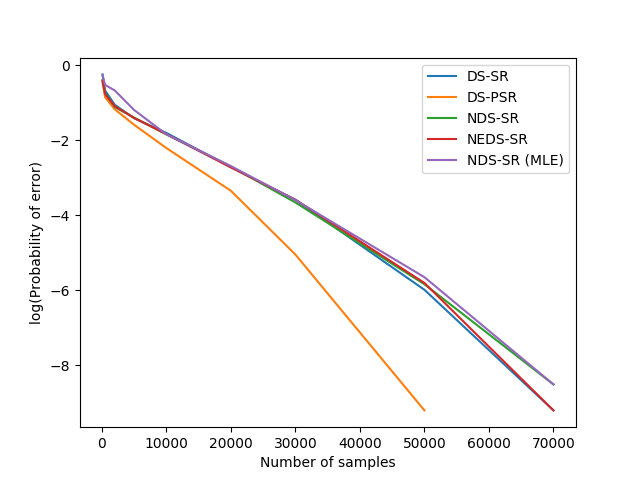}
	\caption{Airbnb NYC housing datas}
	\label{figure:brazil_box}
	\centering
\end{figure}}
Figure~\ref{Fig:Airbnb} shows the performance of the various algorithms on the Airbnb Apartment Listing dataset. Here again, DS-PSR performs the best since it allocates queries in proportion to box sizes. However, unlike the previous dataset, all the other algorithms have comparable performance. This includes DS-SR which does not use any box size information and is still able to perform better since the box sizes are relatively closer to each other for this dataset and the number of communities in each box are also fewer which makes it unlikely that the box containing the largest community is eliminated. 
\ignore{
In the Airbnb Apartment Listing Dataset similar to the brazil Real Estate Dataset both A4 and A7 which allocate samples in proportional to the box size perform well, 
A3 which does that same but uses MLE of the box size to allocate budget instead of the actual box size does not perform as well. \\
However in this case A8 converges and performs as well as A2 unlike the brazil Real Estate Dataset due to the box sizes being relatively close to each other.

\begin{figure}[H]
	\centering
	\includegraphics[scale=0.5]{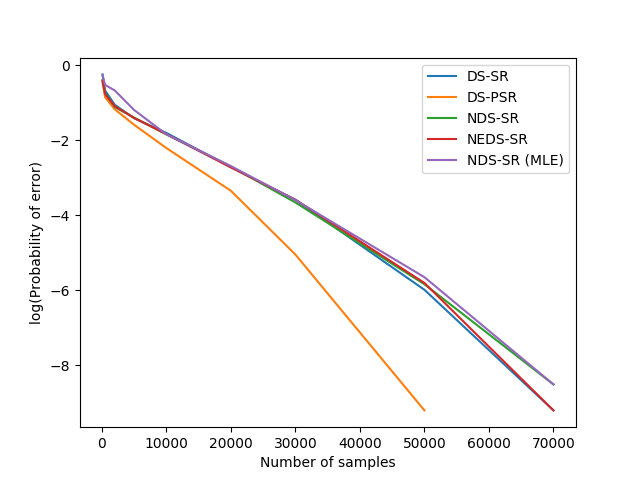}
	\caption{P(error) vs t for the Airbnb Apartment Listing Dataset}
	\label{figure:airbnb_box}
	\centering
\end{figure}}
%
%\nk{Should we first plot only those estimators which don't use box size and those which have perfect knowledge. We can then on a separate plot see the performance of estimators which learn the box size and maybe compare it to the one with perfect knowledge.}
%\jk{Minor: y-axis ticks need to be made consistent across figs. The first few plots the true value of $P_e$ on log scale, while the last few plot $\log P_e$ (also on log scale).} 
%\jk{If possible, might be good to re-plot 3(a) with more runs, so that the separation between algorithms is cleaner.}

\subsection{General Setting Mode Estimation}
\begin{figure}[t]
	\centering
	\includegraphics[scale=0.5]{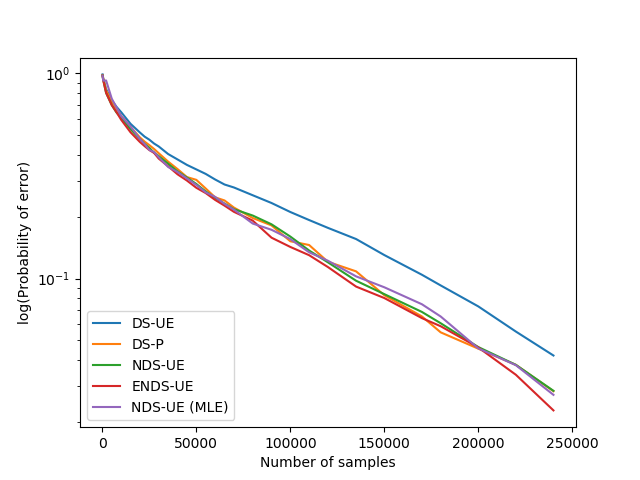}
	\caption{$\log(P_e(D))$ vs $t$ for the Youtube Video Dataset, General Box Setting}
	\label{figure:general_box_youtube}
	\centering
\end{figure}

Finally, we consider the general setting where individuals in a community can be spread across multiple boxes. Section~\ref{sec:general} described various single-round algorithms for this setting, namely the Distinct Samples Uniform Exploration (DS-UE) which doesn't need any box size information and divides the query budget equally among all boxes; the Distinct Samples Proportional Exploration (DS-PE) which assigns queries in proportion to the box sizes; and the various  normalized single phase variants of DS-UE, which we refer to as NDS-UE, ENDS-UE and NDS-UE (MLE). To compare the performance of these different estimators under the general setting, we use the following dataset. 
\begin{itemize}
    \item Trending Youtube Video Statistics Dataset \cite{youtube_dataset}: This dataset contains the top trending videos for different regions such as Canada, US, and Japan, out of which we consider six regions. Mapped to our framework, a region corresponds to a box, a channel denotes a community, and each video represents an individual entity. The goal is to find the most popular channel which has the largest number of trending videos across the six regions. Note that a particular channel (community) can have trending videos (individuals) spread across different regions (boxes) and thus this dataset corresponds to the general setting. This dataset contains 239662 videos, each associated with one of 17773 channels. Top 5 channels have [870, 809, 752, 717, 712] top trending videos across regions. The boxes have comparable size, given by  [40881, 40840, 40724, 38916, 37352, 40949]. 
\end{itemize}
Figure~\ref{figure:general_box_youtube} shows the performance of the various algorithms on the above dataset. Note that all the estimators are able to achieve an exponential decay in the probability of error with the query budget even in this general setting. Furthermore, here the rate of decay for all the estimators is comparable since the box sizes are all similar and thus the knowledge of box sizes does not provide a distinct advantage. However, in terms of the absolute value, DS-UE performs slightly poorly as compared to the other algorithms which either use prior knowledge of box sizes or learn estimates for them using samples.

\appendix

\section{Proof of Theorem \ref{theorem:mixed_identityless_ub}}
\label{sec:proof_sample_freq_estimate}

%\begin{proof} 
Let $\hat{\mu}_i(t)$ be the number of samples seen from $C_i$ over the horizon. We have
\begin{align*}
    \hat{\mu}_i(t) = \sum_{j=1}^t \mathds{1}_{\text{\{person j $\in$ $C_i$\}}}\\
    \Rightarrow E[\hat{\mu}_{i}(t)] = \mu_{i}(t) = \frac{td_i}{N}.
\end{align*}
Using the union bound on $P_e{(D)}$, we get
\begin{align*}
    P_e{(D)} \leq \sum_{i=2}^{m}P(\hat{\mu}_{i}(t)-\hat{\mu}_{1}(t) \geq 0).
\end{align*}
The Chernoff bound gives us
\begin{align*}
    P\left(\hat{\mu}_{k}(t)-\hat{\mu}_{1}(t) - \left(\mu_{k}(t)-\mu_{1}(t)\right) \geq w\right) &\leq \min\limits_{\lambda > 0}e^{-\lambda w}E\left[e^{\lambda(\hat{\mu}_{k}(t)-\hat{\mu}_{1}(t) - (\mu_{k}(t)-\mu_{1}(t))}\right] \\&= \min\limits_{\lambda > 0}e^{-\lambda[w+(\mu_{k}(t)-\mu_{1}(1))]}\left[\frac{d_{k}e^{\lambda}}{N}+\frac{d_{1}e^{-\lambda}}{N}+\left(1-\frac{d_1+d_k}{N}\right)\right]^t.
\end{align*}
Choosing $w = \mu_{1}(t)-\mu_{k}(t)$ and minimizing over $\lambda$,
\begin{align}
\label{mixed:identityless_ub_1}
    P(\hat{\mu}_{k}(t)-\hat{\mu}_{1}(t) \geq 0) \leq \left[1-\frac{(\sqrt{d_1}-\sqrt{d_k})^2}{N}\right]^t
\end{align}
\begin{align*}
    \Rightarrow P_e{(D)} \leq \sum_{i=2}^{m}P(\hat{\mu}_{i}(t)-\hat{\mu}_{1}(t) \geq 0) &\leq \sum_{i=2}^{m} \left[1-\frac{(\sqrt{d_1}-\sqrt{d_k})^2}{N}\right]^t\\
    &\leq (m-1)\left[1-\frac{(\sqrt{d_1}-\sqrt{d_2})^2}{N}\right]^t.
\end{align*}
%\end{proof}

\ignore{Since $d_1>d_k$, we have $\frac{(\sqrt{d_1}-\sqrt{d_k})^2}{N} > \frac{(d_1-d_k)^2}{4Nd_1}$.
\begin{align}
\label{mixed:identityless_ub_2}
    P(\hat{\mu}_{k}(t)-\hat{\mu}_{1}(t) \geq 0) \leq exp\left(-\frac{t(d_1-d_k)^2}{4Nd_1}\right)
\end{align}
\begin{align}
\label{mixed:identityless_ub_3}
    \Rightarrow P_e{(D)} \leq \sum_{i=2}^{m}P(\hat{\mu}_{i}(t)-\hat{\mu}_{1}(t) \geq 0) &\leq \sum_{i=2}^{m} exp\left(-\frac{t(d_1-d_i)^2}{4Nd_1}\right)\\
    &\leq (m-1)exp\left(-\frac{t(d_1-d_2)^2}{4Nd_1}\right)
\end{align}
The intermediate inequality in the theorem statement is obtained if the bound used in (\ref{mixed:identityless_ub_3}) is replaced with (\ref{mixed:identityless_ub_1}) instead of (\ref{mixed:identityless_ub_2}).
}%end ignore 

%%%%%%%%%%%%%%%%%%%%%%%%%%%%%%%%%%%%%%%%%%%%%%%%%
%%%%%%%%%%%%%%%%%%%%%%%%%%%%%%%%%%%%%%%%%%%%%%%%%
%%%%%%%%%%%%%%%%%%%%%%%%%%%%%%%%%%%%%%%%%%%%%%%%%

\section{Proof of Theorem \ref{theorem:mixed_identityless_lb_1}}
\label{sec:proof_identityless_consistent_lb}

%\begin{proof} 
\newtext{To prove the theorem, we consider two instances $D = (d_1,d_2,\ldots,d_m)$ and $D'=(d_1',d_2',\ldots,d_m')$, where the optimal community in $D$ is $C_1$ and the optimal community in $D'$ is $C_2$. We note that the mixed community setting can be modelled as a probability distribution over communities, with the probability of sampling $C_i$ under $D$ and $D'$ being $p_i = d_i / N$ and $p_i' = d_i' / N$ respectively. Let the probability distributions corresponding to instances $D$ and $D'$ be $\Theta = (p_1, p_2, ... p_m)$ and $\Theta' = (p_1', p_2', ... p_m')$ respectively. Further, let the sequence of $t$ samples be denoted by $X_1, X_2, \ldots, X_t$ where $X_i$ is the index of the community that is sampled at time $i$, and  let $\mathbb{P}_{\Theta}, \mathbb{P}_{\Theta'}$ denote the probability measures induced on the sample sequence by the instances $D$, $D'$. Next, we state a few lemmas which will help in the proof of the theorem.}
%, Let our consistent algorithm be $\mathcal{A}$, and let $\mathbb{P}_{D}$ be the probability law on the output of $\mathcal{A}$ under $D$. 
%and let $\mathbb{P}_{\Theta}$ be the probability law on the output of $\mathcal{A}$ under instance $D$. .
\begin{lemma}
\label{lemma:mixed_identityless_lb_proof}
For every event $\mathcal{E} \in F_t$, where $F_t = \sigma(X_1, X_2, ... X_t)$,
\begin{align*}
    \mathbb{P}_{\Theta'}(\mathcal{E}) = \mathbb{E}_\Theta[\mathds{1}_\mathcal{E} \exp(-L_t)],
\end{align*}
\newtext{where $L_t = \sum_{i=1}^{t}\log\left(\frac{p_{X_i}}{p'_{X_i}}\right)$ and $\mathds{1}$ is the indicator random variable.}
%$p$ and $p'$ denote the probability \newtext{distributions corresponding to two different instances}, $D$ and $D'$, with $p(X_i) = p_i, p'(X_i) = p'_i$.
\end{lemma}
\begin{proof} This is analogous to \cite[Lemma 18]{Kaufmann16a}. \end{proof}
\begin{lemma}
\label{lemma:mixed_identityless_lb_proof_2}
For every event $\mathcal{E} \in F_t$,
\begin{align*}
    \mathbb{E}_\Theta[L_t|\mathcal{E}] \geq \log\frac{\mathbb{P}_\Theta(\mathcal{E})}{\mathbb{P}_{\Theta'}(\mathcal{E})}.
\end{align*} 
\end{lemma}
\begin{proof} From Lemma \ref{lemma:mixed_identityless_lb_proof}, we know that $\mathbb{P}_{\Theta'}(\mathcal{E}) = \mathbb{E}_{\Theta}[\exp(-L_t)\mathds{1}_\mathcal{E}]$. Then, using Jensen's inequality on $\exp(-x)$, we have that
\begin{align*}
    \mathbb{P}_{\Theta'}(\mathcal{E}) = \mathbb{E}_{\Theta}[\exp(-L_t)\mathds{1}_\mathcal{E}] = \mathbb{E}_{\Theta}[\mathbb{E}_{\Theta}[\exp(-L_t)|\mathds{1}_\mathcal{E}]\mathds{1}_\mathcal{E}] \geq \mathbb{E}_{\Theta}[\exp(-\mathbb{E}_\Theta[L_t|\mathcal{E}])\mathds{1}_\mathcal{E}]
    \\ = \exp(-\mathbb{E}_\Theta[L_t|\mathcal{E}])\mathbb{P}_\Theta(\mathcal{E})
\end{align*}
The last line above proves the lemma. 
\end{proof}
\begin{lemma}
If $d(x, y) = x\log\left(\frac{x}{y}\right)+(1-x)\log\left(\frac{(1-x)}{(1-y)}\right)$, then for every event $\mathcal{E} \in F_t$, 
\begin{align*}
    \mathbb{E}_{\Theta'}[-L_t] \geq d(\mathbb{P}_{\Theta'}(\mathcal{E}), \mathbb{P}_{\Theta}(\mathcal{E})).
\end{align*}
\end{lemma}
\begin{proof} From Lemma \ref{lemma:mixed_identityless_lb_proof_2} we know that 
\newtext{
\begin{align*}
    \mathbb{E}_{\Theta'}[-L_t|\mathcal{E}] \geq \log\left(\frac{\mathbb{P}_{\Theta'}(\mathcal{E})}{\mathbb{P}_{\Theta}(\mathcal{E})}\right), \mathbb{E}_{\Theta'}[-L_t|\mathcal{E}^c] \geq \log\left(\frac{\mathbb{P}_{\Theta'}(\mathcal{E}^c)}{\mathbb{P}_{\Theta}(\mathcal{E}^c)}\right).
\end{align*}
Using the total law of probability and the above inequality, we get
\begin{align*}
    \mathbb{E}_{\Theta'}[-L_t] = \mathbb{E}_{\Theta'}[-L_t|\mathcal{E}]\mathbb{P}_{\Theta'}(\mathcal{E}) + \mathbb{E}_{\Theta'}[-L_t|\mathcal{E}^c]\mathbb{P}_{\Theta'}(\mathcal{E}^c) \geq d(\mathbb{P}_{\Theta'}(\mathcal{E}), \mathbb{P}_{\Theta}(\mathcal{E}^c)).
\end{align*}}
\end{proof}
%
%To prove the theorem, we consider two distributions $D$ and $D'$, where the optimal community in $D$ is $C_1$ and the optimal community in $D'$ is $C_2$. 
Consider a consistent algorithm $\mathcal{A}$, and let $P_e(D)$ and $P_e(D')$ denote the  probabilities of error for $\mathcal{A}$ under the instances $D$ and $D'$ respectively. Denote the community that is output by $\mathcal{A}$ as $\hat{h}^*$, and let $S$ be the event that $\hat{h}^*=1$. Thus, $P_e(D) = 1-\mathbb{P}_{\Theta}(S)$ and $P_e(D') \geq \mathbb{P}_{\Theta'}(S)$. \newtext{Since algorithm $\mathcal{A}$ is consistent and thus its probability of error on both $D, D'$ goes to zero as the number of samples $t$ grows large}, we have that for every $\epsilon>0$ there exists $t_0(\epsilon)$ such that for all $t \geq t_0(\epsilon), \mathbb{P}_{\Theta'}(S) \leq \epsilon \leq \mathbb{P}_{\Theta}(S)$. For $t \geq t_0(\epsilon)$, 
\begin{align*}
    \mathbb{E}_{\Theta'}[-L_t] \geq d(\mathbb{P}_{\Theta'}(S), \mathbb{P}_{\Theta}(S)) \geq d(\epsilon, \mathbb{P}_{\Theta}(S)) \geq \epsilon\log\left(\frac{\epsilon}{\mathbb{P}_{\Theta}(S)}\right) + (1-\epsilon)\log\left(\frac{(1-\epsilon)}{P_e(D)}\right) \\\geq \epsilon\log(\epsilon) +  (1-\epsilon)\log\left(\frac{(1-\epsilon)}{P_e(D)}\right)
\end{align*}
Taking the limsup, using $\mathbb{E}_{\Theta'}[-L_t] = t.D(\Theta' || \Theta)$ where $D(\cdot || \cdot)$ denotes the Kullback-Leibler divergence, and letting $\epsilon \rightarrow 0$, we get 
\begin{align*}
    \limsup_{t\rightarrow\infty} -\frac{1}{t}\log(P_e(D)) \leq D(\Theta' || \Theta).
\end{align*}
%\nk{I think we have used $D(\cdot || \cdot)$ elsewhere to denote KL-divergence. Should check everywhere and make it consistent.}

Consider $\Theta = (p_1, p_2, ... p_m)$ and $\Theta' = (\frac{\sqrt{p_1p_2}-\delta}{C}, \frac{\sqrt{p_1p_2}+\delta}{C}, \frac{p_3}{C}, ... \frac{p_m}{C})$, where $C = 1 - (\sqrt{p_1}-\sqrt{p_2})^2$ \newtext{and $\delta > 0$ is sufficiently small so that~$\Theta'$ is a probability distribution}. Then, we get
\begin{gather*}
    \limsup_{t\rightarrow\infty} -\frac{1}{t}\log(P_e(D)) \leq \log\left(\frac{1}{C}\right)+\left(\frac{\sqrt{p_1p_2}-\delta}{C}\right)\log\left(\frac{\sqrt{p_1p_2}-\delta}{p_1}\right)+\left(\frac{\sqrt{p_1p_2}+\delta}{C}\right)\log\left(\frac{\sqrt{p_1p_2}+\delta}{p_2}\right)\\\
    \implies \limsup_{t\rightarrow\infty} -\frac{1}{t}\log(P_e(D)) \leq \log\left(\frac{1}{C}\right) \text{ (letting $\delta \downarrow 0$).}
\end{gather*}

%end ignore

%%%%%%%%%%%%%%%%%%%%%%%%%%%%%%%%%%%%%%%%%%%%%%%%%
%%%%%%%%%%%%%%%%%%%%%%%%%%%%%%%%%%%%%%%%%%%%%%%%%
%%%%%%%%%%%%%%%%%%%%%%%%%%%%%%%%%%%%%%%%%%%%%%%%%

\section{Proof of Theorem \ref{theorem:mixed_identity_ub}}
\label{sec:proof_mixed_identity_ub}
%\nk{In the main text, this theorem has only two assertions whereas here there are three. Only retain the proof for the two stated bounds.}

\ignore{ %hiding the theorem statement
\begin{theorem}
For any instance $D = (d_1, d_2, ... d_m),$ where $d_1 > d_2 \geq d_3 \geq \cdots \geq d_m,$ the Distinct Samples Maximization (DSM) algorithm 
%given in Algorithm~\ref{alg:distinctsamplefrequency} 
has a probability of error upper bounded as
\begin{align}
    &P_e(D) \leq 2(m-1)\exp\left(-\frac{t\left(d_1-\frac{\sum_{i=2}^{m}d_i}{m-1}\right)^2}{32Nd_1}\right)\quad \text{ for } t\leq min\left\{\frac{d_1+d_m}{2d_1}N, \frac{16Nd_1}{(d_1-d_m)^2}\right\}, \\
    &P_e(D) \leq 2(m-1)exp\left(-\frac{t(d_1-d_2)^2}{32Nd_1}\right) \quad \text{ for } t \leq \frac{d_1+d_2}{2d_1}N \\
    &P_e(D) \leq {\binom{d_1}{d_2}} \left(1-\frac{d_1-d_2}{N}\right)^t {= {\binom{d_1}{d_2}}\exp\left(-t\log\left(\frac{N}{N-d_1+d_2}\right)\right)} 
    %\exp{\left(-\frac{t(d_1-d_2)}{N}\right)} 
    \quad \forall t.
\end{align}
\end{theorem}
} %end ignore

%\begin{proof}
We will begin by proving the first assertion in the theorem statement which provides an upper bound on the probability of error for $t\leq \min\left\{\frac{d_1+d_m}{2d_1}N, \frac{16Nd_1}{(d_1-d_m)^2}\right\}$. Let $S_i(t)$ denote the number of distinct samples seen from community $C_i$ in $t$ samples. We have the following lemma:
\begin{lemma}
\label{lemma:mixed_identity_distinct_samples_bound} The probability of error of the DSM algorithm is bounded as
\begin{align*}
    P_e(D) \leq \sum_{i=2}^{m}P(S_i(t)-S_1(t) > 0) + \frac{1}{2}P(S_i(t) = S_1(t)).
\end{align*}
\end{lemma}

\newtext{
\begin{proof}
For any $i \in 2,3,\ldots, m$, it is clear that when $S_i(t) - S_1(t) > 0$, DSM will erroneously output $i$ as the index of the community mode. Furthermore, since DSM breaks ties arbitrarily, with some positive probability (bounded by $1/2$) it makes the same error when $S_i(t) = S_1(t)$. Together with the union bound over all $i \in 2,3,\ldots, m$, this gives the above result. 
%there is a chance of $\hat{h}^* = 1$ (because our algorithm breaks ties arbitrarily). Thus, the factor of $\frac{1}{2}$ appears.
\end{proof}
}

Next, for each $k \in \{2,3,\ldots,m\}$ let $Z_k$ be the random variable denoting the number of samples observed from communities $C_1$ and $C_k$.\footnote{Note that $Z_k$ corresponds to the total number of samples from communities $C_1$ and $C_k$, not necessarily distinct.} We note that the expected value of $Z_k$ is given by 
\begin{align}
\label{eqn:mixedexpectation}
    E[Z_k] = \frac{(d_1+d_k)t}{N}.
\end{align}
Define events $E_{k1} = \{Z_k \in [(1-\epsilon_k)E[Z_k],(1+\epsilon_k)E[Z_k]]\}$ and $E_{k2} = E_{k1}^c$, with 
\begin{equation}
\label{Eqn:Epsilon}
\epsilon_k = \frac{\sqrt{\frac{9}{64}\beta_k^4+\frac{3}{2}\beta_k^2}-\frac{3}{8}\beta_k^2}{2} \mbox{ where } \beta_k = \frac{d_1-d_k}{d_1+d_k}. 
\end{equation}
It is easy to verify that $\beta_k < 1$ and $\epsilon_k \le \min\{\beta_k, 1/2\}$. 
Then, we have
\begin{align}
\nonumber
  & \ P(S_k(t) - S_1(t) > 0) + \frac{1}{2}P(S_k(t) = S_1(t)) \\
  \nonumber
  \leq & \ P(S_k(t)-S_1(t) > 0 | E_{k1}) P(E_{k1}) + P(S_k(t)-S_1(t) > 0 | E_{k2}) P(E_{k2}) \\
  \nonumber
  &\hspace{1.75in}+ \frac{1}{2}P(S_k(t)=S_1(t)|E_{k1})P(E_{k1}) + \frac{1}{2}P(S_k(t)=S_1(t)|E_{k2})P(E_{k2})\\
  \leq & \ P(S_k(t)-S_1(t) \geq 0 | E_{k1}) P(E_{k1}) + P(S_k(t)-S_1(t) > 0 | E_{k2}) P(E_{k2}) 
    + \frac{1}{2}P(S_k(t)=S_1(t)|E_{k2})P(E_{k2}).
    \label{Eqn:decomp}
\end{align}
Note that the LHS above appears for each $k \in \{2,3,\ldots,m\}$ in the upper bound on $P_e(D)$ in Lemma~\ref{lemma:mixed_identity_distinct_samples_bound}. We will bound the terms in the RHS separately, and then combine them together to get an overall upper bound on $P_e(D)$. To begin with, note that
%\begin{equation*}
 %   S_i(t) = \sum_{j=1}^{d_i}\mathds{1}_{\{\text{person $j$ from community $C_i$ is seen at least once in $t$ samples}\}}
%\end{equation*}
\begin{equation}
\label{eqn:mixedcommexpectation}
    E[S_i(t)|Z_k] = d_i \left[1-\left(1-\frac{1}{d_1+d_k}\right)^{Z_k}\right], \text{for $i \in \{1,k\}$}.
\end{equation}
\newtext{We consider the function $f(x_1, x_2, x_3, ..., x_t) = S_k(t) - S_1(t)$ where $x_i$ is the identity of the individual sampled at the $i$-th instant. Note that for any $i \in \{1,2,\ldots,t\}$ and for all $x_1, x_2, x_3, ..., x_t, x_i' \in \{1,2,\ldots,N\}$, we have $|f(x_1, x_2, ...,x_i, ..., x_t) - f(x_1, x_2, ...,x_i', ..., x_t)| \le c_i \triangleq 2\mathds{1}_{x_i \ \in \ C_1 \cup C_k}$}. Then, conditioning on $Z_k$ and applying McDiarmid's inequality, we get
%, and $S_i(t) = \sum_{j=1}^{t}\mathds{1}_{y_j = i}$
\begin{align*}
    P(f-E[f|Z_k] \geq t'|Z_k) \leq P(|f-E[f|Z_k]| \geq t' | Z_k) &\leq \mathrm{exp}\left(-\frac{2t'^2}{\sum_{i=1}^{t}c_{i}^2}\right) = \newtext{\mathrm{exp}\left(-\frac{t'^2}{2Z_k}\right)}.
\end{align*}
Plugging in $t' = -E[f|Z_k]$, and computing $E[f|Z_k]$ using Equation (\ref{eqn:mixedcommexpectation}), we obtain
\begin{align}
    \label{eqn:mixedterm1bound}
    P(f \geq 0|Z_k) = P(S_k(t)-S_1(t) \geq 0 | Z_k) &\leq exp\left(-\frac{(d_1-d_k)^2\left[1-\left(1-\frac{1}{d_1+d_k}\right)^{Z_k}\right]^2}{2Z_k}\right).
\end{align}
We will start with deriving an upper bound on the first term in the RHS of equation (\ref{Eqn:decomp}) given by $P(S_k(t)-S_1(t) \geq 0 | E_{k1})P(E_{k1})$. Conditioned on the event $E_{k1}$, we have $Z_k \in [(1-\epsilon_k)E[Z_k],(1+\epsilon_k)E[Z_k]]$. Furthermore, from the statement of the first part of the theorem statement and the definitions of $\epsilon_k, \beta_k$ from equation (\ref{Eqn:Epsilon}), we have the following sequence of assertions: 
\begin{equation*}
    t \leq \frac{d_1+d_k}{2d_1}N \Rightarrow \beta_k = \frac{d_1-d_k}{d_1+d_k} \leq \frac{N}{t}-1 \Rightarrow \epsilon_k \le  \frac{N}{t}-1 \Rightarrow Z_k \leq (1+\epsilon_k)\frac{t(d_1+d_k)}{N} \leq d_1+d_k .
%   \item $\epsilon_k \leq \beta_k$ if $\beta_k \leq 1 \Rightarrow \epsilon_k \leq \beta_k = \frac{d_1-d_k}{d_1+d_k} \leq \frac{N}{t} - 1 \Rightarrow Z_k \leq (1+\epsilon_k)\frac{t(d_1+d_k)}{N} \leq d_1+d_k$ 
\end{equation*}
Using the above inequalities and the Taylor series expansion, we have 
\begin{align}
    \label{eqn:mixedtaylorbound}
    \left[1-\left(1-\frac{1}{d_1+d_k}\right)^{Z_k}\right] &\geq \left[\frac{Z_k}{d_1+d_k} - \frac{{Z_k}^2}{2(d_1+d_k)^2}\right] \geq \frac{Z_k}{2(d_1+d_k)} .
\end{align}
Plugging the bound above in equation (\ref{eqn:mixedterm1bound}), and using $Z_k \geq (1-\epsilon_k)E[Z_k] = (1-\epsilon_k)(d_1+d_k)t/N,$ we have 
\begin{align}
\label{Eqn:Term1}
    P(S_k(t)-S_1(t) \geq 0 | E_{k1}) \times P(E_{k1}) \le P(S_k(t)-S_1(t) \geq 0 | E_{k1}) \leq exp\left(-\frac{t(1-\epsilon_k)(d_1-d_k)^2}{8N(d_1+d_k)}\right),
\end{align}
thus giving us an upper bound on the first term in the RHS of equation (\ref{Eqn:decomp}). \\
%exp\left(-\frac{Z_k(d_1-d_k)^2}{8(d_1+d_k)^2}\right)
%Bounding $P(E_{k1})$ by $1$, {\color{blue}using equation} (\ref{eqn:mixedexpectation}) and using $Z_k \geq (1-\epsilon_k)E[Z_k]$, we obtain the bound on term 1 as
%\begin{align}
%\label{Eqn:Term1}
 %   P(S_k(t)) - P(S_1(t)) \geq 0 | E_{k1}) \times P(E_{k1}) &\leq exp\left(-\frac{t(1-\epsilon_k)(d_1-d_k)^2}{8N(d_1+d_k)}\right).
%P(|Z_k-E[Z_k]|\leq \epsilon_k E[Z_k])    
%\end{align}

For bounding the sum of the second and third terms in the RHS of equation (\ref{Eqn:decomp}), we use the following lemma:
\begin{lemma}
\label{lemma:mixed_identity_distinct_samples_bound_2}
\newtext{For any $k \in \{2,3,\ldots,m\}$ so that $d_k \leq d_1$ and for any $l \ge 0$}, we have
\begin{align*}
    P(S_k(t)-S_1(t) > 0 | Z_k = l) + \frac{1}{2}P(S_k(t)=S_1(t)|Z_k = l) \leq \frac{1}{2}
\end{align*}
\end{lemma}
\newtext{
\begin{proof}
Note that the theorem statement is equivalent to showing that, when $d_k \leq d_1$, 
\begin{align*}
    P(S_k(t)-S_1(t) > 0 | Z_k = l) \leq P(S_k(t)-S_1(t) < 0 | Z_k = l),
\end{align*}
which says that, conditioned on the total number of samples from communities $1$ and $k$ together being some fixed $l$, the likely event is that the community $1$, whose size is at least that of community $k$, will have as many or more distinct individuals than community $k$. Given $d_k \leq d_1$, this is intuitive and while it can be argued formally, we skip the argument here for brevity.   
\end{proof}
}

Using Lemma~\ref{lemma:mixed_identity_distinct_samples_bound_2}, we get that the second and third terms in the RHS of equation (\ref{Eqn:decomp}) are bounded as 
\begin{align*}
    P(S_k(t)-S_1(t) > 0 | E_{k2}) P(E_{k2}) + \frac{1}{2}P(S_k(t)=S_1(t)|E_{k2})P(E_{k2}) \leq \frac{1}{2}P(E_{k2}). 
\end{align*}
Further, using Chernoff's inequality for $P(E_{k2})$ and $E[Z_k] = (d_1+d_k)t/N$, we have
\begin{align}
\label{Eqn:Term2}
    \frac{1}{2}P(E_{k2}) = \frac{1}{2}P(|Z_k-E[Z_k]| > \epsilon_k) \leq \mathrm{exp}\left(-\frac{\epsilon_k^2(d_1+d_k)t}{3N}\right). 
\end{align}
Finally, combining Lemma~\ref{lemma:mixed_identity_distinct_samples_bound}, equation (\ref{Eqn:Term1}), and equation (\ref{Eqn:Term2}), we get the following upper bound on $P_e(D)$. 
\begin{align*}
    P_e(D) \leq \sum_{k=2}^{m} \mathrm{exp}\left(-\frac{t(1-\epsilon_k)(d_1-d_k)^2}{8N(d_1+d_k)}\right) + \mathrm{exp}\left(-\frac{\epsilon_k^2(d_1+d_k)t}{3N}\right). 
\end{align*}
From the value of $\epsilon_k$ in equation (\ref{Eqn:Epsilon}), we have that the exponents in the two terms of the summation above are equal. Thus, we have 
\begin{align}
\label{Eqn:ProbError}
    P_e(D) \leq \sum_{k=2}^{m} 2\mathrm{exp}\left(-\frac{t(1-\epsilon_k)(d_1-d_k)^2}{8N(d_1+d_k)}\right) \leq \sum_{k=2}^{m} 2\mathrm{exp}\left(-\frac{t(d_1-d_k)^2}{16N(d_1+d_k)}\right) \leq \sum_{k=2}^{m} 2\mathrm{exp}\left(-\frac{t(d_1-d_k)^2}{32Nd_1}\right),
\end{align}
where the first inequality is true because $\epsilon_k \le 1/2$; and the second inequality follows since $d_k \leq d_1$ for all $k \in \{2,3,\ldots,m\}$.
%is an increasing function of $\beta_k \leq 1$, thus \\$\epsilon_k \leq 0.453 \leq 0.5$

The next result comments on the shape of the function $f(x) = \mathrm{exp}(-\frac{t(d_1-x)^2}{32Nd_1})$, which appears in equation (\ref{Eqn:ProbError}) above.
\begin{lemma}
\label{lemma:mixed_identity_distinct_samples_bound_3}
The function $f(x) = \mathrm{exp}\left(-\frac{t(d_1-x)^2}{32Nd_1}\right)$ is concave for any \newtext{$x \ge d_m$} and $t \leq \frac{16Nd_1}{(d_1-d_m)^2}$.
%, and thus Jensen's inequality can be used in this region.
\end{lemma}
\begin{proof} We differentiate $f(x)$ twice to confirm that it is concave.
\begin{align*}
    f''(x) = \frac{t}{16Nd_1}\mathrm{exp}\left(-\frac{t(d_1-x)^2}{32Nd_1}\right)\left(\frac{t}{16Nd_1}(d_1-x)^2-1\right)
\end{align*}
Using the inequality $t \leq \frac{16Nd_1}{(d_1-d_m)^2}$, we have that
\begin{align*}
    f''(x) \leq \frac{t}{16Nd_1}\mathrm{exp}\left(-\frac{t(d_1-x)^2}{32Nd_1}\right)\left(\frac{(d_1-x)^2}{(d_1-d_m)^2}-1\right)
\end{align*}
which implies $f''(x)\leq0$ since $x \geq d_m$. 
\end{proof}
From (\ref{Eqn:ProbError}) and using Lemma~\ref{lemma:mixed_identity_distinct_samples_bound_3}, we have from Jensen's inequality that for $t \leq \min\left\{\frac{d_1+d_m}{2d_1}N, \frac{16Nd_1}{(d_1-d_m)^2}\right\}$
\begin{align*}
    P_e(D) \leq 2 \sum_{k=2}^{m} \mathrm{exp}\left(-\frac{t(d_1-d_k)^2}{32Nd_1}\right) \le 2(m-1)\mathrm{exp}\left(-\frac{t\left(d_1-\frac{\sum_{k=2}^{m}d_i}{m-1}\right)^2}{32Nd_1}\right),
\end{align*}
which proves the first assertion in the theorem statement. 
\remove{

the fact that this function is concave in the region $d_m \leq x \leq d_2$, we use Jensen's inequality in the following form:

\begin{align*}
    \sum_{i=2}^{m}f(x_i) \leq (m-1)f\left(\frac{\sum_{i=2}^{m}x_i}{m-1}\right),
\end{align*}
to get that 
\begin{align*}
    P_e(D) \leq 2(m-1)exp\left(-\frac{t\left(d_1-\frac{\sum_{i=2}^{m}d_i}{m-1}\right)^2}{32Nd_1}\right).
\end{align*}
as required. We note that, for proving concavity we used the fact that $t \leq \frac{16Nd_1}{(d_1-d_m)^2}$, and to show the upper bound for each term from $i = 2,3,...m$ we used that $t \leq \frac{d_1+d_m}{2d_1}$. }\\\\
%{\color{red}If we instead do not use concavity at all, and simply use the bound for $d_2$ for each of the terms, we get, for $t \leq \frac{d_1+d_2}{2d_1}N$ the following bound holds:
%\begin{align*}
 %   P_e(D) \leq 2(m-1)exp\left(-\frac{t(d_1-d_2)^2}{32Nd_1}\right)
%\end{align*}
%This gives the second assertion in the theorem statement. 
%
\newtext{For the second assertion in the theorem statement, note that the algorithm will certainly not make an error if the number of distinct individuals seen from the $i$-th community, $S_i(t) \geq d_2+1$, where $d_2$ denotes the size of the second-largest community. Hence, the probability of error is bounded as $P_e(D) \le P(S_1(t) \leq d_2)$. Further, note that if the event $\{S_1(t) \leq d_2\}$ occurs, then there exists a set of $d_1-d_2$ individuals in $C_1$ which remain unsampled in the $t$ samples. Thus, we have 
\begin{gather*}
    P_e(D) \leq P(S_1(t)\leq d_2) \leq \binom{d_1}{d_2}\left(1-\frac{d_1-d_2}{N}\right)^t.
\end{gather*}
}

\section{Proof of Theorem \ref{theorem:mixed_identity_lb}}
\label{sec:proof_mixed_identity_lb}

%\begin{proof} 
This proof is similar in spirit to the proof of \cite[Theorem 1]{Moulos19}. Consider an instance $D = (d_1, d_2, \ldots, d_m).$ First, we note that since  
%with $d_i$'s denoting the sizes of the underlying communities and let $N = \sum_i d_i$ represent the total number of individuals in the box. As before, we will assume that community $1$ is the largest community in the instance $D$. 
$(S_j(t)),\ 1\leq j \leq m)$ is a sufficient statistic for~$D,$ it suffices to restrict attention to (consistent) algorithms whose output depends only on the vector $(S_j(t),\ 1 \leq j \leq m).$ Given this restriction, we track the temporal evolution of the vector $S(k) = (S_j(k),\ 1 \leq j \leq m),$ where $S_j(k)$ is the number of distinct individuals from community~$j$ seen in the first~$k$ oracle queries.
%Given the restriction on the class of algorithms we consider, it suffices to only keep track of the number of distinct individuals seen from each community, as we get the query responses from the oracle. 
This evolution can be modeled as an absorbing Markov chain over state space $\prod_{j=1}^m \{0,1,\cdots d_i\},$ with $S(0) = (0,0,\cdots,0).$
%$\{X_j\}_{j=1}^{t}$, where each state in the underlying state space is associated with a vector $(s_1, s_2, ..., s_m)$ with $s_i$ denoting the number of distinct individuals seen so far from community $i$. More formally, the state space is given by $\mathcal{S} = \{(s_1, s_2, ..., s_m): 0 \le s_j \le d_j \forall j, \sum_{j} s_j \le t\}$ where recall that $t$ denotes the query budget. Note that the state of the Markov process at time $0$ is given by $X_0 = (0,0,\ldots,0)$. An example diagram of the Markov chain with $m=2$ communities is shown in Figure~\ref{figure:vector_markov_chain}.\\
%
\ignore{
\begin{figure}[!ht]
	\centering
	\includegraphics[width=0.8\textwidth]{vector_markov.png}
	\caption{Markov Chain with state vectors}
	\label{figure:vector_markov_chain}
	\centering
\end{figure}
}
Next, let us write down the transition probabilities $q_{D}(s,s')$ for each state pair $(s,s').$ Note that from state~$s,$ the chain can transition to the states $s + e_j$ for $1 \leq j \leq m,$ where the vector~$e_j$ has 1 in the $j$th position and 0 elsewhere, or remain in state~$s.$ Moreover, $q_D(s,s+e_j) = (d_j-s_j)/N,$ and $q_D(s,s) = \frac{\sum_{j=1}^m s_j}{N}.$ 

%Suppose we are in state $x = (s_1, s_2, ..., s_m)$ and let $x_i = (s_1, s_2, ..., s_i+1, ..., s_m)$ for each $i \in [1:m]$. Note that after the next query, we can only go from state $x$ to itself or to a state amongst $\{x_i\}_{i=1}^{m}$. We have a self-transition at state $x$ if in the next query, the sampled individual has already been seen before, and thus the probability $q_{D}(x,x)$ of a self-transition at state $x$ is given by  $\sum_{i}s_i / N$. On the other hand, for each $i \in [1:m]$, we have the probability $q_{D}(x,x_i)$ of going from $x$ to $x_i$ is given by $(d_i-s_i)^+/N$.\\

Recall that by assumption, community $1$ is the largest community for the instance $D.$ Let us consider an alternate instance $D' = (d_1', d_2', \ldots, d_m')$ such that $d_1' = d_2 - 1$, $d_j' = d_j \ \forall j \neq 1$, and $N' = N - d_1 + d_2 - 1$. Note that the community mode under the alternate instance $D'$ is different from that under the original instance $D$. 
%Furthermore, the transition probabilities of the Markov process at state $x = (s_1, s_2, ..., s_m)$ under the alternate instance $D'$ are given by $q_{D'}(x,x) = \sum_{i} s_i / N'$ and $q_{D'}(x,x_i) = (d_i'-s_i)^+/N'$. 
Thus, for state $s$ that is feasible under both $D$ and $D'$, 
$$\log\left(\frac{q_{D'}(s,s)}{q_{D}(s,s)}\right) = \log\left(\frac{N}{N-d_{1}+d_{2}-1}\right).$$ Similarly, for state pair $(s,s+e_j)$ that is feasible under both $D$ and $D'$,
\begin{align*}
%    \log\left(\frac{q_{D'}(x,x)}{q_{D}(x,x)}\right) = \log\left(\frac{N}{N-d_{1}+d_{2}-1}\right) \\ 
    \log\left(\frac{q_{D'}(s,s+e_j)}{q_{D}(s,s+e_j)}\right) &= \log\left(\frac{N}{N-d_{1}+d_{2}-1}\right), j \neq 1, \\
    \log\left(\frac{q_{D'}(s,s+e_1)}{q_{D}(s,s+e_1)}\right) &= \log\left(\frac{N(d_2-1-s_1)}{(N-d_{1}+d_{2}-1)(d_1-s_1)}\right) = \log\left(\frac{N}{N-d_{1}+d_{2}-1}\right) + \log\left(\frac{d_2-1-s_1}{d_1-s_1}\right).
\end{align*}
Therefore, for any state pair $(s,s')$ such that $q_D(s,s'),q_{D'}(s,s') > 0,$ we have 
\begin{equation}
    \label{Eq:LLUB}
    \log\left(\frac{q_{D'}(s,s')}{q_{D}(s,s')}\right) \leq \log\left(\frac{N}{N-d_{1}+d_{2}-1}\right).
\end{equation}

Next, let $\mathbb{P}_{D}, \mathbb{P}_{D'}$ denote the probability measures induced by the algorithm under consideration under the instances $D$ and $D',$ respectively. %\sj{\mathbb{P} also depends on an algorithm $\mathcal{A}$ - we should show this dependence?} \jk{It does not, since in the mixed community setting, the sampling is not governed by yhe algorithm. Is that right?}
Then, given a state evolution sequence $(S(1),\cdots,S(t))$, the log-likelihood ratio is given by 
\begin{align*}
\log\frac{\mathbb{P}_{D'}(S(1),\cdots,S(t)) }{\mathbb{P}_{D}(S(1),\cdots,S(t))} = \sum_{s,s'} N(s,s',t)\log\left(\frac{q_{D'}(s,s')}{q_{D}(s,s')}\right),
\end{align*}
where $N(s,s',t)$ represents the number of times the transition from state $s$ to state $s$ occurs over the course of $t$ queries. Combining with \eqref{Eq:LLUB}, we get
\begin{align*}
\log\frac{\mathbb{P}_{D'}(S(1),\cdots,S(t)) }{\mathbb{P}_{D}(S(1),\cdots,S(t))} \leq t \log\left(\frac{N}{N-d_{1}+d_{2}-1}\right),
\end{align*}
which implies
\begin{equation}
\label{Eq:KL-UB}
  D(\mathbb{P}_{D'}||\mathbb{P}_{D}) =   E_{D'}\left[\log\frac{\mathbb{P}_{D'}(S(1),\cdots,S(t))}{\mathbb{P}_{D}(S(1),\cdots,S(t))}\right] \leq t \log\left(\frac{N}{N-d_{1}+d_{2}-1}\right),
\end{equation}
where $D(\cdot || \cdot)$ denotes the Kullback-Leibler divergence. 
%We note that $E_{D'}[N(u, v, 0, t)]$ summed across all state-next state pairs is exactly equal to $t$. Thus, we have 
On the other hand, since the algorithm produces an estimate $\hat{h}^*$ of the community mode based solely on $S(t)$, we have from the data-processing inequality (see \cite{ITBook}) that
\begin{equation}
\label{Eq:DataProc}
D(\mathbb{P}_{D'}||\mathbb{P}_{D}) \ge D\bigl(Ber(\mathbb{P}_{D'}(\hat{h}^* = 1))||Ber(\mathbb{P}_{D}(\hat{h}^* = 1))\bigr),
\end{equation}
where $Ber(x)$ denotes the Bernoulli distribution with parameter $x \in (0, 1)$. Recall that the community mode under $D$ is community~$1$, while it is community~2 under $D'$. Then from the definition of consistent algorithms, for every $\epsilon > 0,$  $\exists$ $t_0(\epsilon)$ such that for $t \geq t_0(\epsilon), \mathbb{P}_{D'}(\hat{h}^* = 1) \leq \epsilon \leq \mathbb{P}_{D}(\hat{h}^* = 1)$. Thus, we have 
\begin{align*}
     &D(Ber(\mathbb{P}_{D'}(\hat{h}^* = 1))||Ber(\mathbb{P}_{D}(\hat{h}^* = 1))) \geq D(Ber(\epsilon)||Ber(\mathbb{P}_{D}(\hat{h}^* = 1))) \\ &\quad \geq \epsilon\log\left(\frac{\epsilon}{\mathbb{P}_{D}(\hat{h}^* = 1)}\right) + (1-\epsilon)\log\left(\frac{1-\epsilon}{\mathbb{P}_{D}(\hat{h}^* \neq 1)}\right) \geq \epsilon\log(\epsilon) + (1-\epsilon)\log\left(\frac{1-\epsilon}{\mathbb{P}_{D}(\hat{h}^* \neq 1)}\right).
\end{align*}
Using $\epsilon \rightarrow 0$ and $\mathbb{P}_{D}(\hat{h}^* \neq 1) = P_e(D)$, we have \ignore{\sj{Should we be writing the below statement? I thought we only wanted to make the infimum statement as $\epsilon \rightarrow 0$}}
\begin{align*}
    D(Ber(\mathbb{P}_{D'}(\hat{h}^* = 1))||Ber(\mathbb{P}_{D}(\hat{h}^* = 1))) \geq -\log(P_e(D)).
\end{align*}
Finally, combining with \eqref{Eq:KL-UB} and \eqref{Eq:DataProc}, we have that
\begin{gather*}
    \liminf_{t \ra \infty} \frac{\log(P_e(D))}{t} \geq - \log\left(\frac{N}{N-(d_1-d_2+1)}\right).
\end{gather*}
%
% for $t$ large enough
% \begin{align*}
%     -\log(P_e(D)) \leq t\log\left(\frac{N}{N-d_1+d_2-1}\right) \\
%     \implies P_e(D) \geq \exp\left(-t\log\left(\frac{N}{N-d_1+d_2-1}\right)\right) \geq \exp\left(-\frac{t(d_1-d_2+1)}{N-d_1+d_2-1}\right) 
% \end{align*}
%
% where the last inequality follows from $\log(1+x) \le x$ $\forall x > 0$.
%\end{proof}

%%%%%%%%%%%%%%%%%%%%%%%%%%%%%%%%%%%%%%%%%%%%%%
%%%%%%%%%%%%%%%%%%%%%%%%%%%%%%%%%%%%%%%%%%%%%%
%%%%%%%%%%%%%%%%%%%%%%%%%%%%%%%%%%%%%%%%%%%%%%

\section{Proof of Theorem \ref{theorem:collision_audibert}}
\label{sec:proof_collision_audibert}

Note that
\begin{align*}
    P_e(D) & \leq \sum_{r=1}^{b-1} P(\text{$C_1$ gets eliminated in round~$r$}). 
\end{align*}    
Let $S_i(K)$ denote the number of (immediate pairwise) collisions recorded in $C_i$ after $K$ pairs of samples. Since at least one of the smallest~$r$ communities is guaranteed to be present during round~$r,$
\begin{align}
    P_e(D) & \leq \sum_{r=1}^{b-1} \sum_{j=b+1-r}^{b} P(S_j(K_r)-S_1(K_r) \leq 0) \nonumber \\
    & \leq \sum_{r=1}^{b-1} r P(S_{b+1-r}(K_r)-S_1(K_r) \leq 0).
    \label{eq:collision_pe_bound1}
\end{align}
Denoting, for $i \neq 1,$ $f_i(K) := S_{i}(K)-S_1(K),$ we now derive an upper bound on $P(f_i(K) \leq 0).$ Applying Chernoff's inequality, for $\lambda \leq 0,$
\begin{align*}
    P(f_{i}(K)\leq0) &\leq E\left[e^{\lambda f_{i}(K) }\right]\\
    &= \biggl[\frac{1}{d_1d_i}+\left(1-\frac{1}{d_1}\right)\left(1-\frac{1}{d_i}\right)+e^\lambda\left(1-\frac{1}{d_1}\right)\frac{1}{d_i}+e^{-\lambda}\left(1-\frac{1}{d_i}\right)\frac{1}{d_1}\biggr]^{K}.\\
\end{align*}
Setting $e^\lambda = \sqrt{\frac{d_i-1}{d_1-1}}$,
\begin{align*}
    P(f_{i}(K) \leq 0) &\leq \left(1-\frac{(\sqrt{d_1-1}-\sqrt{d_i-1})^2}{d_1d_i}\right)^{K} \leq \mathrm{exp}\left(-\frac{K(\sqrt{d_1-1}-\sqrt{d_i-1})^2}{d_1d_i}\right).
\end{align*}
Since $d_1>d_i$, $(\sqrt{d_1-1}-\sqrt{d_i-1})^2>\frac{((d_1-1)-(d_i-1))^2}{4(d_1-1)}>\frac{((d_1-1)-(d_i-1))^2}{4d_1} = \frac{(d_1-d_i)^2}{4d_1}$.
\begin{align*}
    \Rightarrow P(f_{i}(K) \leq 0) \leq \mathrm{exp}\left(-\frac{{K}(d_1-d_i)^2}{4d_1^2d_i}\right).
\end{align*}

Substituting the above into \eqref{eq:collision_pe_bound1}, 
\begin{align*}
    P_e(D)
    &\leq \sum_{r=1}^{b-1} r\  \mathrm{exp}\left(-\frac{K_r(d_1-d_{b+1-r})^2}{4d_1^2d_{b+1-r}}\right).
\end{align*}
Since $K_r = \left\lceil\frac{1}{\overline{log}(b)}\frac{t/2-b}{b+1-r}\right\rceil$, where $\overline{log}(b) = \frac{1}{2}+\sum_{i=2}^b\frac{1}{i}$ and $\Delta_{i} = \frac{1}{d_i}-\frac{1}{d_1}$,
\begin{align*}
    P_e(D) \leq \sum_{r=1}^{b-1} r\  \mathrm{exp}\left(-\frac{{K_rd_{b+1-r}\Delta_{(b+1-r)}^2}}{4}\right).
\end{align*}
For ${H}^c(D) = \underset{i\in [2:b]}{max}\frac{i\Delta_{i}^{-2}}{d_i}$,
\begin{align*}
    K_rd_{b+1-r}\Delta_{(b+1-r)}^2\geq \frac{(t/2-b)}{\overline{log}(b){H}^c(D)}
\end{align*}
\begin{align*}
    \Rightarrow P_e(D) \leq \frac{b(b-1)}{2}\mathrm{exp}\left(-\frac{(t/2-b)}{4\overline{log}(b){H}^c(D)}\right).
\end{align*}

%%%%%%%%%%%%%%%%%%%%%%%%%%%%%%%%%%%%%%%%%%%%%%%%
%%%%%%%%%%%%%%%%%%%%%%%%%%%%%%%%%%%%%%%%%%%%%%%%
%%%%%%%%%%%%%%%%%%%%%%%%%%%%%%%%%%%%%%%%%%%%%%%%

\section{Proof of Theorem \ref{Thm:UBBox}}

\label{sec:proof_ub_box}

Let $P^i_e(D)$ denote the probability of the community mode being eliminated at the $i$th step; i.e, for $i \leq b-1, P^i_e(D)$ denotes the probability of removing box $1$ in phase~$i$ of SR, and $P^b_e(D)$ denotes the probability of choosing the wrong community from box~1 after this box survived the~$(b-1)$ SR phases. Then, we have
\begin{align*}
    P_e(D) &= \sum_{i=1}^{b-1} P^i_e(D) + P^b_e(D),\\
    P^i_e(D) &\leq \binom{d_{11}}{c_{b-i+1}}\exp\left(-K_i\log\left(\frac{N_1}{N_1-d_{11}+c_{b-i+1}}\right)\right) \quad (1 \leq i \leq b-1),\\
    P^b_e(D) &\leq \binom{d_{11}}{c_{1}}\exp\left(-K_{b-1}\log\left(\frac{N_1}{N_1-d_{11}+c_{1}}\right)\right),
\end{align*}
where the second and third statements are based on a coupon collector argument, similar to the one employed in the proof of Theorem~\ref{thm:distinctsamplesAudibert} for the separated community setting. The proof is now completed by substituting the values of $K_r,$ and using the definition of $H^b(D).$

%%%%%%%%%%%%%%%%%%%%%%%%%%%%%%%%%%%%%%%%%%
%%%%%%%%%%%%%%%%%%%%%%%%%%%%%%%%%%%%%%%%%%
%%%%%%%%%%%%%%%%%%%%%%%%%%%%%%%%%%%%%%%%%%

\section{Proof of Theorem~\ref{Thm:ENDS-SR}}
\label{sec:NEDS_decay_rate}

We show that ENDS-SR has the same decay rate as DS-SR. Recall that the comparison function used in ENDS-SR is 
\begin{gather*}
    \frac{S_{ij}N_i}{E[S_i]},
\end{gather*}
where $S_{ij}$ is the number of distinct samples from community $i$ in box $j$, and $S_i$ is the number of distinct samples from box $i$. At the end of $r$ rounds,   
\begin{gather*}
    E[S_i] = N_i\left(1-\left(1-\frac{1}{N_i}\right)^{K_r}\right).
\end{gather*}
Similar to the coupon collector argument in the proof of Theorem~\ref{Thm:UBBox}, we let $P^i_e(D)$ be the probability of the community mode being eliminated in the $i$th step. We have that
\begin{gather*}
    P_e(D) \leq \sum_{i=1}^b P_e^i(D).
\end{gather*}
After $r \leq b-1$ rounds/phases, the comparison function for the largest community equals
\begin{gather*}
    \frac{S_{11}}{\left(1-\left(1-\frac{1}{N_1}\right)^{K_r}\right)}.
\end{gather*}
For some community $j$ in box $i$, the comparison function is 
\begin{gather*}
    \frac{S_{ij}}{\left(1-\left(1-\frac{1}{N_i}\right)^{K_r}\right)} \leq \frac{c_{i}}{\left(1-\left(1-\frac{1}{N_m}\right)^{K_r}\right)},
\end{gather*}
where $N_m = \max_{i} N_i$. Thus, if we have 
\begin{gather*}
    S_{11} > \frac{c_{b-r+1}\left(1-\left(1-\frac{1}{N_1}\right)^{K_r}\right)}{\left(1-\left(1-\frac{1}{N_m}\right)^{K_r}\right)},
\end{gather*}
then the community mode cannot be eliminated in the $r$th round. For round $r = b$, we just note that 
\begin{gather*}
    S_{11} > c_1
\end{gather*}
is sufficient for the community mode estimate to be correct. Applying the coupon collector argument on these events, by using the notation
\begin{gather*}
    f_{i}(K) := \frac{c_{i}\left(1-\left(1-\frac{1}{N_1}\right)^{K}\right)}{\left(1-\left(1-\frac{1}{N_m}\right)^{K}\right)},
\end{gather*}
we have
\begin{gather*}
    P_e(D) \leq \sum_{i=1}^{b-1} \binom{d_{11}}{f_{b-i+1}(K_i)}\exp\left(-K_i\log\left(\frac{N_1}{N_1-d_{11}+f_{b-i+1}(K_i)}\right)\right) + \binom{d_{11}}{c_1}\exp\left(-K_b\log\left(\frac{N_1}{N_1-d_{11}+c_1}\right)\right).
\end{gather*}
We note that, as $t \rightarrow \infty$, $f_i(t) \rightarrow c_i,$ which then implies the statement of the theorem.

%%%%%%%%%%%%%%%%%%%%%%%%%%%%%%%%%%%%%%%%%%%%%%%%%%
%%%%%%%%%%%%%%%%%%%%%%%%%%%%%%%%%%%%%%%%%%%%%%%%%%
%%%%%%%%%%%%%%%%%%%%%%%%%%%%%%%%%%%%%%%%%%%%%%%%%%

\section{Proof of Theorem \ref{Thm:LBBox2}}
\label{sec:proof_lb_box}

We first state the following lemma (analogous to Lemma~\ref{lemma:expected_samples_ub_separated}) for this setting (the proof is straightforward and omitted):
\begin{lemma}
\label{lemma:expected_samples_ub_box_1}
For any algorithm $\mathcal{A}$ and instance $D$, there must exist a box $a \in [2:b]$ such that $E_D[N_a(t)] \leq \frac{t}{(\log(N_1)-\log(N_1-d_{11}+c_a))H_2^b(D)}$, where $N_a(t)$ denotes the number of times box $a$ is sampled in $t$ queries under $\mathcal{A}$.
\end{lemma}
%\begin{proof} Similar to the proof of Lemma $\ref{lemma:expected_samples_ub_separated}$, this follows from assuming the contrary and showing a contradiction. \ignore{\jk{Why is this true? The summation in the definition of $H_2$ runs from 1 to $b,$ not from 2 to $b.$}}
%\end{proof}
%
\begin{proof}[Proof of Theorem~\ref{Thm:LBBox2}]
Given an instance $D$, we construct an alternate instance $D^{[a]}$ by changing the size of the largest community in box $a$ (corresponding to the one specified by Lemma \ref{lemma:expected_samples_ub_box_1}) from $c_a$ to  %$g_a' = c_a+\left(\left(\frac{N_1}{N_1-d_{11}+c_{a}}\right)^\Gamma-1\right)N_a,$ 
$g_a' = c_a+N'_a -N_a.$
%where $\Gamma$ is as defined in the statement of the theorem.
\footnote{We use $g'_a$ and not $c'_a$ to denote the new size of this community because in the alternate instance $D^{[a]}$, this community is the largest community, and is thus no longer the \emph{competing} community in box~$a.$} Note that the size of box $a$ changes from $N_a$ to $N_a' = N_a + g_a' - c_a.$

Furthermore, we can see that the community mode under instance $D^{[a]}$ is different from the one under the original instance $D$, since 
\ignore{$$
g_a' = c_a+ N'_a - N_a \geq c_a + \left( \frac{N_1(N_a + d_{11})}{N_a(N_1 - d_{11}+ c_b)} - 1\right)N_a > c_a + (N_a + d_{11}) - N_a >  d_{11}.
$$}
$$
g_a' = c_a+ N'_a - N_a \geq c_a +  \frac{N_1(N_a -c_a + d_{11})}{(N_1 - d_{11}+ c_a)} - N_a > c_a + (N_a - c_a + d_{11}) - N_a = d_{11}.
$$
Following steps similar to the proof of Theorem \ref{theorem:lower_bound_separated}, we get
\begin{gather*}
    D(\mathbb{P}_D, \mathbb{P}_{D^{[a]}}) \leq E_D[N_a(t)]\log\left(\frac{N_a'}{N_a}\right).
\end{gather*}
From the definition of $\Gamma,$ it follows that $\frac{N_a'}{N_a} = \left(\frac{N_1}{N_1-d_{11}+c_a}\right)^\Gamma.$ Thus, invoking Lemma~\ref{lemma:expected_samples_ub_box_1}, we have
\begin{gather*}
    D(\mathbb{P}_D, \mathbb{P}_{D^{[a]}}) \leq \frac{t\Gamma}{H^b_2(D)}.
\end{gather*}
Finally, similar to the proof of Theorem \ref{theorem:lower_bound_separated}, we use Lemma \ref{lemma:mixed_identityless_lb_proof_3} to get
\begin{gather*}
    \max\left(P_e(D), P_e(D^{[a]})\right) \geq \frac{1}{4}\exp\left(-\frac{t\Gamma}{H_2^b(D)}\right)
\end{gather*}
which matches the statement of the theorem. \\

\remove{
We have an initial distribution $D$, and want to clear an alternate distribution $v_b^{[a]}$ such that the largest community exists in the $a^{th}$ box (which is the specific box for which Lemma \ref{lemma:expected_samples_ub_box} holds). \\\\
Working along the lines of Theorem \ref{theorem:lower_bound_separated}, we have that
\begin{gather*}
    D(\mathbb{P}_D, \mathbb{P}_{D_b^{[a]}}) \leq E[N_a(t)]\log\left(\frac{N_a'}{N_a}\right)
\end{gather*}
Setting $\frac{N_a'}{N_a} = \left(\frac{N_1}{N_1-d_{11}+g_a}\right)^k$, we have that
\begin{gather*}
    D(\mathbb{P}_D, \mathbb{P}_{D_b^{[a]}}) \leq \frac{tk}{H^b_2(D)}
\end{gather*}
We note that we want $g_a' > d_{11}$ and $H_2^b(D_b^{[a]}) \leq H_2^b(D)$. We show that these hold with our choice of $N_a'$. 
\begin{gather*}
    g_a' > d_{11} \\ 
    \implies N_a'+g_a > d_{11} + N_a \\
    \implies \frac{N_1(N_a+d_{11})}{N_1-d_{11}} + g_a > (N_a+d_{11}) \\
    \implies \frac{N_1}{N_1-d_{11}}+\frac{g_a}{N_a+d_{11}} > 1
\end{gather*}
where the last statement is clearly true.\\\\
}
Finally, we show that 
\begin{gather*}
    H_2^b(D^{[a]}) \leq H_2^b(D)
    \Leftrightarrow \sum_{i \neq a} \frac{1}{\log(N_a')-\log(N_a'-g_a'+c_i')} \leq \sum_{i\neq 1} \frac{1}{\log(N_1)-\log(N_1-d_{11}+c_i)}
\end{gather*}
\newtext{We do this in two steps:
\begin{enumerate}
    \item Firstly, for each $i \notin \{1, a\}$, we show that the  term corresponding to box $i$ in the sum on the left is smaller than the corresponding term in the sum on the right, i.e., 
    \begin{align*}
    \frac{1}{\log(N_a')-\log(N_a'-g_a'+c_i')} &\le \frac{1}{\log(N_1)-\log(N_1-d_{11}+c_i)} \\
    \mbox{or equivalently, }\quad  \frac{N_1}{N_1-d_{11}+c_i} &\leq \frac{N_a'}{N_a'-g_a'+c_i'}. 
    \end{align*}
    This follows from the following sequence of inequalities. 
    \begin{align*}
 \frac{N_1}{(N_1-d_{11}+c_i)}
 = \frac{N_1(N_a - c_a + c_i)}{(N_1-d_{11}+c_i)(N_a- c_a + c_i)} \leq \frac{N_a'}{N_a-c_a+c_i} = \frac{N_a'}{N_a'-g_a'+ c_i'} 
 \end{align*}
 where the last step follows since $N_a' = N_a + g_a' - c_a$ and $c_i' = c_i$ for $i \notin \{1, a\}$.
 \item Secondly, we show that the term corresponding to box $1$ in the sum on the left is smaller than the term corresponding to box $a$ in the sum on the right, i.e,
\begin{align*}
\frac{1}{\log(N_a')-\log(N_a'-g_a'+c_1')} &\leq \frac{1}{\log(N_1)-\log(N_1-d_{11}+c_a)} \\
\mbox{or equivalently, }\quad \frac{N_1}{(N_1-d_{11}+c_a)} &\leq \frac{N_a'}{N_a'-g_a'+d_{11}} . 
\end{align*}
This follows from the following sequence of inequalities. 
\begin{align*}
    \frac{N_1}{N_1-d_{11}+c_a} = \frac{N_1(N_a-c_a+d_{11})}{(N_a-c_a+d_{11})(N_1-d_{11}+c_a)} \leq \frac{N_a'}{N_a-c_a+d_{11}} = \frac{N_a'}{N_a'-g_a'+d_{11}},
 \end{align*}
 where the last step is true because $N_a' - g_a' = N_a - c_a$.
\end{enumerate}
This completes the proof. }
\end{proof}

\section{Other Lemmas}
\label{sec:other_lemmas}

\begin{lemma}
\label{lemma:mixed_identityless_lb_proof_3}
Let $\rho_0$ and $\rho_1$ be two probability distributions supported on some set $\chi$, with $\rho_1$ absolutely continuous with respect to $\rho_0$. Then for any measurable function $\phi : \chi \rightarrow \{0,1\}$, 
\begin{align*}
    P_{X\sim\rho_0}(\phi(X)=1)+P_{X\sim\rho_1}(\phi(X)=0) \geq \frac{1}{2}\exp\left(-D(\rho_0 || \rho_1)\right)
\end{align*}
\end{lemma}
\begin{proof} This is \cite[Lemma 20]{Kaufmann16a}. \end{proof}

\begin{lemma}
$\frac{H(D)}{2} \leq H_2(D) \leq \overline{log}(b)H(D).$
\end{lemma}
\begin{proof} 
For the inequality on the left, we note that
\begin{align*}
    H_2(D) = \sum_{i=2}^{b}\frac{1}{\log(d_1)-\log(d_i)} \geq \sum_{i=2}^{j}\frac{1}{\log(d_1)-\log(d_i)} \geq \frac{j-1}{\log(d_1)-\log(d_j)} \forall j \in[2:b]
\end{align*}
Since this is true for all $j \in [2:b]$, taking the max of these values and using that $j-1 \geq \frac{j}{2}, j \geq 2$ we have
\begin{align*}
    H_2(D) \geq \max_{j{\neq 1}} \frac{j/2}{\log(d_1)-\log(d_j)} = \frac{H(D)}{2}
\end{align*}
For the inequality on the right, we multiply and divide each term in the summation of $H_2(D)$ by $i$:
\begin{align*}
    H_2(D) = \sum_{i=2}^{b}\frac{i}{i(\log(d_1)-\log(d_i))} \leq \sum_{i=2}^{b}\frac{H(D)}{i} \leq \overline{log}(b)H(D)
\end{align*}
This completes the proof of both inequalities in the statement of the lemma.
\end{proof}

 \bibliographystyle{elsarticle-num} 
 \bibliography{cas-refs}

%% else use the following coding to input the bibitems directly in the
%% TeX file.

% \begin{thebibliography}{00}

% %% \bibitem{label}
% %% Text of bibliographic item

% \bibitem{}

% \end{thebibliography}
\end{document}